\documentclass[letterpaper]{article} 
\usepackage{aaai2026}  
\usepackage{times}  
\usepackage{helvet}  
\usepackage{courier}  
\usepackage[hyphens]{url}  
\usepackage{graphicx} 
\usepackage{natbib}  
\usepackage{caption} 
\usepackage{newfloat}
\usepackage{listings}
\DeclareCaptionStyle{ruled}{labelfont=normalfont,labelsep=colon,strut=off} 
\usepackage{float}
\floatstyle{ruled}
\newfloat{listing}{tb}{lst}{}
\floatname{listing}{Listing}
\title{MoSE: Unveiling Structural Patterns in Graphs via Mixture of Subgraph Experts}
\author {
    Junda Ye\thanks{This work was completed while Junda Ye was a visiting PhD student at Nanyang Technological University.}\textsuperscript{\rm 1,\rm 3},
    Zhongbao Zhang\footnotemark[2]\textsuperscript{\rm 1},
    Li Sun\textsuperscript{\rm 2},
    Siqiang Luo\thanks{Co-Corresponding author.}\textsuperscript{\rm 3}
}
\affiliations {
    \textsuperscript{\rm 1}Beijing University of Posts and Telecommunications\\
    \textsuperscript{\rm 2}North China Electric Power University\\
    \textsuperscript{\rm 3}Nanyang Technological University\\
    \{jundaye, zhongbaozb\}@bupt.edu.cn, ccesunli@ncepu.edu.cn, siqiang.luo@ntu.edu.sg
}

\usepackage{bibentry}

\usepackage[linesnumbered, ruled, vlined]{algorithm2e}
\usepackage{mathtools}
\usepackage{amsthm}
\usepackage{amsmath}
\usepackage{amsfonts}
\usepackage{cleveref}
\usepackage{xspace}
\usepackage{xcolor}
\usepackage{bm}
\usepackage{enumerate}
\usepackage{multirow}
\usepackage{colortbl}
\usepackage{booktabs}

\newcommand{\ie}{\textit{i.e.}}

\crefname{equation}{Eq.}{Eqs.}
\crefname{table}{Tab.}{Tabs}
\crefname{figure}{Fig.}{Figs.}
\crefname{section}{Sec.}{Secs.}
\crefname{algorithm}{Algorithm}{Algorithms}
\crefname{proposition}{Prop.}{Props.}
\crefname{corollary}{Cor.}{Cors.}

\newtheorem{lemma}{Lemma}
\newtheorem{definition}{Definition}
\newtheorem{corollary}{Corollary}

\newtheorem{proposition}{Proposition}
\DeclareMathOperator*{\concat}{%
    \mathchoice%
        {\Big\Vert}%
        {\big\Vert}%
        {\Vert}%
        {\Vert}%
}
\newcommand{\multiset}[1]{\left\lbrace \! \left\lbrace {#1} \right\rbrace \! \right\rbrace}

\newcommand{\ourmethod}{MoSE\xspace}

\begin{document}

\maketitle

\begin{abstract}
While graph neural networks (GNNs) have achieved great success in learning from graph-structured data, 
their reliance on local, pairwise message passing restricts their ability to capture complex, high-order subgraph patterns. leading to insufficient structural expressiveness.
Recent efforts have attempted to enhance structural expressiveness by integrating random walk kernels into GNNs.
However, these methods are inherently designed for graph-level tasks, which limits their applicability to other downstream tasks such as node classification. Moreover, their fixed kernel configurations hinder the model’s flexibility in capturing diverse subgraph structures. To address these limitations, this paper proposes a novel \textbf{M}ixture \textbf{o}f \textbf{S}ubgraph \textbf{E}xperts (\textbf{\ourmethod}) framework for flexible and expressive subgraph-based representation learning across diverse graph tasks. Specifically, \ourmethod extracts informative subgraphs via anonymous walks and dynamically routes them to specialized experts based on structural semantics, enabling the model to capture diverse subgraph patterns with improved flexibility and interpretability. We further provide a theoretical analysis of \ourmethod's expressivity within the Subgraph Weisfeiler-Lehman (SWL) Test, proving that it is more powerful than SWL. Extensive experiments, together with visualizations of learned subgraph experts, demonstrate that \ourmethod not only outperforms competitive baselines, but also provides interpretable insights into structural patterns learned by the model.
\end{abstract}

\section{Introduction} \label{introduction}
Graphs, serving as data structures capable of depicting complex relationships between entities, are ubiquitous in real-world scenarios, such as social networks \citep{social1, SCARA, liu2025sigma}, biological molecules \citep{Molecular1, Molecular2}, and recommendation systems \citep{recommendation1, recommendation2}. As a prominent approach for learning representations from graphs, graph neural networks (GNNs) have been widely studied \citep{ChebNet, GCN, GraphSAGE, GAT, GIN} based on the philosophy of message-passing mechanism motivated by the Weisfeiler-Lehman (WL) graph isomorphism test \citep{WLtest_original}. These message passing neural networks (MPNNs) constitute the most prevailing type of GNNs, however, their expressive power is known to be upper bounded by the first-order Weisfeiler-Lehman (1-WL) isomorphism test \citep{GIN}. Moreover, recent research has demonstrated that such 1-WL equivalent GNNs lack sufficient expressive power to differentiate fundamental structural patterns such as triangles and cycles \citep{lackpower1, lackpower2}. 

As another line of work, graph kernel-based GNNs have attracted increasing attention in recent years due to their ability to bridge the provable theoretical guarantees of graph kernels with the expressive power of MPNNs \citep{SCKN, KernelNN, DDGK, GNTK, GCKN, GKNN}. Among them, methods based on random walk kernel (RWK) have garnered particular interest due to their strong interpretability \citep{RWNN, GSKN, RWK, GKNN}. 
These RWK-based GNNs typically compute similarities between the input graph and a set of learnable subgraphs, commonly referred to as hidden graphs, to derive graph-level representations. Each hidden graph serves as a structural probe and offers an interpretable way to capture recurring patterns across different datasets.



\begin{figure}[htbp]
    \centering
    \includegraphics[width=1\linewidth]{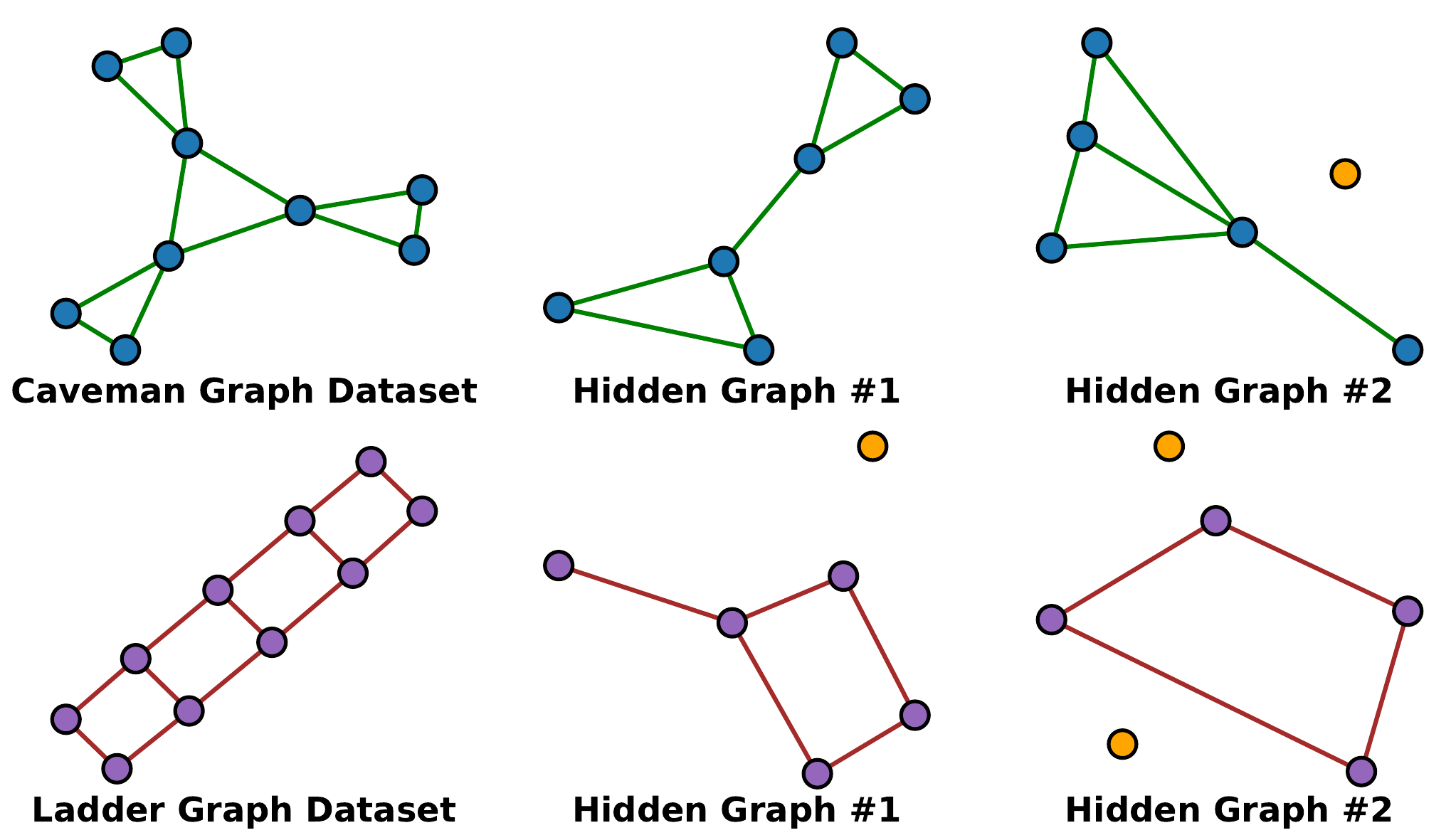}
    \caption{Hidden graphs learned by RWNN on Caveman and Ladder synthetic graph datasets. The hidden graph sizes are all fixed to 6. Redundant nodes, colored in orange, are \textbf{isolated} and contribute little to subgraph patterns.}
    \label{fig:toy_example}
\end{figure}

Although RWK-based methods have achieved promising results, they still suffer from the following limitations. \textit{\textbf{First}}, due to the computational complexity in calculating RWK, existing methods have primarily been constrained to graph-level tasks, which typically involve graphs with relatively fewer nodes. Consequently, these methods have seldom been applied to node-level tasks, thereby severely limiting their applicability to diverse real-world scenarios. While a few recent methods have attempted to extend RWK-based models to node-level tasks \citep{GSKN, RWK}, they compute node representations by invoking RWK calculations or walk simulations over the entire graph, leading to high computational cost. \textit{\textbf{Second}}, current models typically adopt pre-defined hidden graphs and fixed random walk lengths, and they treat each hidden graph equally, resulting in limited flexibility and adaptability. 
This rigid strategy limits the model’s capacity to adaptively represent optimal substructures, leading to redundancy and unnecessary computational overhead.
We illustrate this with a case study in \cref{fig:toy_example}, where fixed hidden graph configurations lead to the inclusion of redundant nodes, ultimately incurring extra computational cost. These insights and limitations raise a natural question: \textbf{\textit{Can we design a more flexible and efficient RWK-based framework that adaptively supports feature and structure modeling for each node to better suit diverse graph tasks and scales while preserving interpretability?}} 

Toward this end, we propose \ourmethod, namely Mixture of Subgraph Experts, a new model grounded in the Mixture of Experts (MoE) paradigm, 
enabling each node to adaptively select subgraph experts from its local context.
Specifically, we introduce a novel subgraph extraction strategy together with a subgraph-aware gating mechanism, which jointly helps subgraph experts to capture informative structural patterns. Compared to other RWK-based models, \ourmethod offers more refined interpretability, and an average 10.84\% performance improvement at around 30\% reduction in runtime. The interpretability is further supported by visualizations of the learned subgraph experts, which reveal compact and diverse structural patterns aligned with task semantics.

The major contributions of our work are summarized as follows:
\begin{itemize}
    \item To the best of our knowledge, we are the first to introduce MoE into RWK-based graph neural networks, offering a flexible framework for modeling diverse subgraph patterns while retaining the semantic interpretability of kernel-based modeling.
    \item We propose a novel model \ourmethod, which integrates a newly designed subgraph extraction strategy based on anonymous walks with a mixture of learnable subgraph experts guided by a subgraph-aware gating mechanism. Grounded in the Subgraph Weisfeiler-Lehman (SWL) test, the model captures discriminative subgraph patterns.
    \item Extensive experiments across 19 datasets on both graph-level and node-level tasks demonstrate the superiority, efficiency, and interpretability of our model over other competitive baselines.
\end{itemize}

\section{Preliminaries}

\subsection{Notation and Problem Formulation}
Let $G = (\mathcal{V}, \mathcal{E})$ denote an undirected, unweighted graph, where $\mathcal{V} = \{v_1, v_2, \dots, v_n\}$ is the set of $n$ nodes and $\mathcal{E} \subseteq \mathcal{V} \times \mathcal{V}$ represents the edge set, respectively. The topological structure of $\mathcal{G}$ can also be represented by a binary adjacency matrix $\bm{A} \in \mathbb{R}^{n \times n}$, where $\bm{A}_{ij}=1$ iff the nodes $v_i$ and $v_j$ are connected, otherwise $\bm{A}_{ij}=0$. Given a node $u\in\mathcal{V}$, denote its neighbors as $\mathcal{N}(u) \coloneq \{v\in \mathcal{V} : \{u, v\} \in \mathcal{E}\}$. We denote the node attributes/features as a matrix $\bm{X} \in \mathbb{R}^{n \times f}$ where $f$ is the feature dimension. $\bm{Y}$ is denoted as the node/graph label set, and each node/graph is assigned to a label $y \in \bm{Y} = \{1, 2, \ldots, C\}$, where $C$ is the number of total classes for node/graph classification task.

In this paper, we focus on two fundamental graph learning tasks: node classification and graph classification. The goal is to learn a representation function that maps each node or an entire graph to a low-dimensional embedding space, such that the learned representations can support accurate downstream classification.

\subsection{Random Walk Kernel}
Random walk kernel (RWK) was originally proposed \citep{original_rwk1, original_rwk2} to measure the similarity between two graphs by counting the number of walks they share. To efficiently compute such walk-based similarity, RWK introduces the concept of direct product graph as follows:

\begin{definition}[\textbf{Direct Product Graph}]
    Given two graphs $G=(\mathcal{V}, \mathcal{E})$ and $G'=(\mathcal{V}', \mathcal{E}')$, their direct product graph $G_\times=(\mathcal{V}_\times, \mathcal{E}_\times)$ is a graph with vertices $\mathcal{V}_\times=\{(v, v') : v\in \mathcal{V} \wedge v' \in \mathcal{V}'\}$ and edges $\mathcal{E}_\times = \{\{(u, u'), (v, v')\} : \{u, v\} \in \mathcal{E} \wedge \{u', v'\} \in \mathcal{E}'\}$.
\end{definition}

It is obvious that $G_\times$ is a graph over paris of nodes from $G$ and $G'$, and two nodes in $G_\times$ are neighbors if and only if the corresponding nodes in $G$ and $G'$ are both neighbors. Therefore, performing a random walk on $G_\times$ is equivalent to performing simultaneous random walks on $G$ and $G'$. Given step $P \in \mathbb{N}$, the $P$-step random walk kernel between two graphs $G$ and $G'$ is defined as:
\begin{equation} \label{p-step random walk kernel}
    \mathcal{K}(G, G')= \sum_{p=0}^{P} \mathcal{K}^{(p)}(G, G') =\sum_{p=0}^{P}  \sum_{i=1}^{\vert \mathcal{V}_\times \vert} \sum_{j=1}^{\vert \mathcal{V}_\times \vert} \lambda_p  [\bm{A}_\times^p]_{ij},
\end{equation}
where $\bm{A}_\times$ is the adjacency matrix of $G_\times$ and $\lambda=(\lambda_0, \lambda_1, \ldots, \lambda_P)$ is a sequence of weights. Note that the $(i, j)$-th entry of $\bm{A}_\times^p$ (\ie,$\bm{A}_\times$ to the power $p$) is the number of common walks of step length $p$ between the $i$-th and $j$-th node in $G_\times$.

\begin{figure*}[ht]
    \centering
    \includegraphics[width=0.88\linewidth]{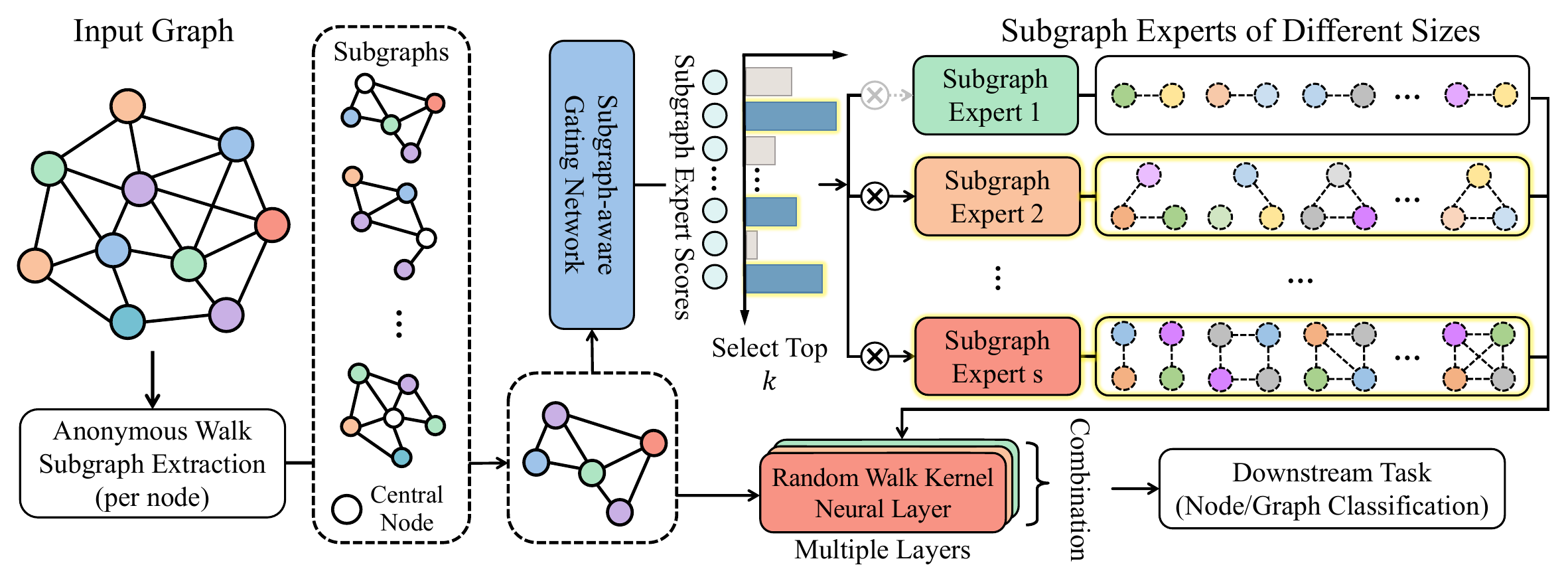}
    \caption{The overall architecture of \ourmethod. The subgraph of each node is extracted via anonymous walks and is routed to suitable subgraph experts (highlighted in yellow) through a subgraph-aware gating network.}
    \label{architecture}
\end{figure*}

To enable integration with neural networks, RWK is extended from the discrete formulation to a continuous version. Let $\bm{S}=\bm{X}'\bm{X}^\top \in \mathbb{R}^{n'\times n}$ be the dot product similarity between the node attributes from two graphs, and we flatten $\bm{S}$ into a vector $\bm{s} :=\mathrm{vec}(\bm{S}) \in \mathbb{R}^{nn'}$ for notational convenience. The differentiable version of \cref{p-step random walk kernel} can then be reformulated as:

\begin{equation} \label{differentiable random walk kernel}
    \mathcal{K}^{(p)}(G, G')= \sum_{i=1}^{\vert \mathcal{V}_\times \vert} \sum_{j=1}^{\vert \mathcal{V}_\times \vert} \bm{s}_i \bm{s}_j [\bm{A}_\times^p]_{ij} = \bm{s}^\top \bm{A}_\times^p \bm{s},
\end{equation}
which facilitates end-to-end learning.

\subsection{Random Walk Kernel Neural Network}
Building on the differentiable formulation of the RWK, recent works have proposed random walk kernel neural networks \citep{RWNN, RWK}, which integrate RWK into end-to-end trainable architectures. These models adopt a set of trainable hidden graphs ${H_i}$ in each layer as structural probes and encode the input graph $G$ by computing its similarity with each hidden graph via the \cref{differentiable random walk kernel}. Formally, the graph-level representation of $G$ is obtained as a concatenation of its RWK-based similarities with all hidden graphs:
\begin{equation} \label{graph level representation}
\begin{aligned}
   \bm{h}(G) = \concat_{i=1}^k \mathcal{K}^{(p)}(G, H_i) 
    =\left[ \mathcal{K}^{(p)}(G, H_1), \ldots, \mathcal{K}^{(p)}(G, H_k)\right],
\end{aligned}
\end{equation}
where $\concat$ is the concatenation operator. The graph representation $\bm{h}(G)$ is then fed into a classifier for prediction.

\section{Proposed Method}
This section details our proposed method named \ourmethod which consists of three key components: a new subgraph extraction strategy, topology-aware expert routing, and subgraph experts. The overall framework is shown in \cref{architecture}.

\subsection{Anonymous Walk-based Subgraph Extraction}
For subgraph-based graph learning methods, the subgraph extraction plays a crucial role in determining the expressiveness of the model \citep{SAGNN}, as the extracted subgraphs define the local context within the node. The most widely adopted strategy is the $h$-hop induced ego graph of node $v$, denoted as $\mathcal{B}(v, h)$.
Many practical implementations restrict themselves to $\mathcal{B}(v, 1)$ since the number of nodes grows exponentially with increasing $h$, leading to substantial computational overhead. Unfortunately, $\mathcal{B}(v, 1)$ is often insufficient for capturing long-range dependencies. To address these limitations, we consider using anonymous walks \citep{AWE}, a compact yet expressive descriptor, as the basis for subgraph extraction.
Unlike random walks, anonymous walks ignore node identities and focus on structure, better capturing frequent subgraph patterns.


\begin{definition}[\textbf{Anonymous Walk}]
Given a random walk $w=(v_0, v_1, \ldots, v_L)$ on a graph, the corresponding anonymous walk is defined as a sequence of integers $a=(\gamma_0, \gamma_1, \ldots, \gamma_L)$, where $\gamma_i=\min\mathrm{pos}(w, v_i)$, and $\mathrm{pos}(w, v_i)$ returns the position at which the node $v_i$ first appears in the walk $w$. We denote the mapping of a random walk $w$ to the anonymous walk $a$ by $w\mapsto a$.
\end{definition}


\begin{proposition} \label{AW reconstruct Ego Graph}
Let $G=(\mathcal{V}, \mathcal{E})$ be a graph and $v \in \mathcal{V}$. A sufficiently rich distribution $\mathcal{D}_l$ over length-$l$ anonymous walks starting at $v$ with $l = \mathcal{O}(|\mathcal{E}'|)$, is sufficient to reconstruct the ego graph $\mathcal{B}(v, h) = (\mathcal{V}', \mathcal{E}')$ centered at $v$ with radius $h$.
\end{proposition}

\cref{AW reconstruct Ego Graph} follows from Theorem 1 in \citep{AWequalKEGO}. It shows that the distribution of anonymous walks rooted at a node contains sufficient information to reconstruct the ego graph around that node within a fixed radius. This motivates us to extract subgraphs by selecting representative anonymous walk patterns instead of relying on $h$-hop ego graph.

We begin by sampling random walks of fixed length starting from the node $v$. The resulting walk set is denoted as $\mathcal{W}_v$. Each walk is then mapped to its anonymous walk, forming the set as: 
\begin{equation} \label{Convert to AW}
    \mathcal{A}_v=\{w\mapsto a: w\in \mathcal{W}_v\}.
\end{equation}

Next, we aggregate all anonymous walks from all nodes and identify the most frequently occuring patterns in the graph. Let $k_{walk}$ be the number of walk patterns to retain. The global top-$k_{walk}$ anonymous walk patterns are defined as follows:
\begin{equation} \label{Select top k pattern from AW}
    \mathcal{P} = \mathrm{TopK}\left(\bigcup_{v\in\mathcal{V}}\mathcal{A}_v, k_{walk} \right).
\end{equation}

For each node $v$, we collect all nodes that appear in random walks whose corresponding anonymous walk belongs to the top patterns $\mathcal{P}$. These nodes form the anonymous walk-based neighborhood $\mathcal{N}_{v}^{aw}$. Finally, we induce the subgraph $G_v$ from the original graph using node $v$ and $\mathcal{N}_{v}^{aw}$:

\begin{equation} \label{induced subgraph based on AW}
    G_v = G[\{v\} \cup \mathcal{N}_{v}^{aw}].
\end{equation}

\subsection{Subgraph-aware Gating Network}
Once the subgraph $G_v=(\mathcal{V}_v, \mathcal{E}_v)$ is extracted for each node $v$, we proceed to assign it to the most appropriate experts. However, existing noisy gating mechanisms \citep{noisy_gating, MoG} typically ignore the underlying graph topology. To address this limitation, we propose a subgraph-aware gating network that conditions expert selection on the topology-aware subgraph $G_v$ as follows:
\begin{align}
    &\bm{\eta} (v) = \sigma \left(\bm{x}_v + \sum\nolimits_{u \in \mathcal{V}_v}\alpha_{uv}\bm{x}_u \right), \label{MoE Gating Network 1}\\
    \bm{\psi}(v) &=\bm{\eta}(v)\bm{W}_g+\epsilon\cdot \mathrm{Softplus}(\bm{\eta}(v)\bm{W}_n), \label{MoE Gating Network 2}\\
    \bm{\zeta}(v) &= \mathrm{Softmax} \left(\mathrm{TopK}\left(\bm{\psi}(v), k_{ept}\right)\right), \label{MoE Gating Network 3}
\end{align}
where $\alpha_{uv}=\exp(\bm{x}_u^\top\bm{x}_v)/\sum_{u\in\mathcal{V}_v}\exp(\bm{x}_u^\top\bm{x}_v)$ are non-parametric attention weights, and $\sigma(\cdot)$ is a nonlinear activation function. The matrices $\bm{W}_g, \bm{W}_n \in \mathbb{R}^{f\times K}$ are trainable parameters that respectively produce clean and noisy scores, with $\epsilon \in \mathcal{N}(0, 1)$ denoting the standard Gaussian noise. The resulting $\bm{\psi}(v) \in \mathbb{R}^K$ are pre-selection logits over all $K$ experts, and $\bm{\zeta}(v) \in \mathbb{R}^{k_{ept}}=[E_1^v, E_2^v, \ldots, E_{k_{ept}}^v]$ denotes the final sparse routing weights over the top-$k_{ept}$ selected experts.

\subsection{Subgraph Experts}
Having selected the proper experts for each node $v$, we now introduce the architecture of each expert and describe how they process the subgraph $G_v$.

We define a group of $K$ experts as $E=\{E_s: s=1, 2, \ldots, K\}$ where each expert $E_s$ consists of $N$ hidden graphs of size $s$, denoted as $E_s=\{H_1^s, H_2^s, \ldots, H_N^s\}$. Each hidden graph $H_i^s$ is equipped with a learnable adjacency matrix $\textsc{ReLU}(\bm{W}_i) \in \mathbb{R}^{s\times s}$ and a learnable feature matrix $\bm{Z}_i \in \mathbb{R}^{s\times f}$.

Given a node $v$ and its extracted subgraph $G_v$ with the corresponding adjacency matrix and node features $\bm{A}_{G_v}$ and $\bm{X}_{G_v}$, each expert computes the node representation through $N$ hidden graphs as:
\begin{equation}\label{simple representation}
    \bm{h}_s(v) = \mathrm{MLP}\left(\concat_{i=1}^N \mathcal{K}^{(p)}\left(G_v, H_i^s\right)\right),
\end{equation}
where $\mathcal{K}^{(p)}\left(G_v, H_i^s\right)$ represents the $p$-step random walk kernel between $G_v$ and $H_i^s$, computed as follows:
\begin{equation} \label{derivation of K}
    \begin{aligned}
        \mathcal{K}^{(p)}(G_v, &H_i) = \bm{s}^\top (\bm{A}_{G_v}^p \otimes \textsc{ReLU}(\bm{W}_i)^p)\bm{s} \\
        &=\bm{s}^\top \mathrm{vec}(\textsc{ReLU}(\bm{W}_i)^p \mathrm{vec}^{-1}(\bm{s})\bm{A}_{G_v}^p) \\
        &= \bm{1}^\top [\bm{Z}_i \bm{X}_{G_v}^\top \odot \textsc{ReLU}(\bm{W}_i)^p \bm{Z}_i \bm{X}_{G_v}^\top \bm{A}_{G_v}^p] \bm{1}.
    \end{aligned}
\end{equation}

After obtaining the expert-specific node representation $\bm{h}_s(v)$, we combine them into a final representation for node $v$ based on the routing weights $\bm{\zeta}(v)$ produced by \cref{MoE Gating Network 3}. Specifically, the final representation is computed as:
\begin{equation} \label{final representation}
    \bm{h}(v) = \textsc{Comb}\left(\{\bm{h}_s(v) : s\in\mathcal{M}\}, \bm{\zeta}\left(v\right)\right),
\end{equation}
where $\mathcal{M}$ denotes the set of selected expert indices and $\textsc{Comb}(\cdot)$ can either be a weighted sum $\sum_{m\in\mathcal{M}} E_m^v\bm{h}_m(v)$ or concatenation form $\mathrm{MLP}(\concat_{m\in\mathcal{M}} E_m^v \bm{h}_m\left(v\right))$. 

For graph-level tasks, we apply a readout function over all node embeddings:
\begin{equation} \label{Graph representation}
    \bm{h}(G)=\textsc{Readout}(\{\bm{h}(v) | v \in \mathcal{V}\}),
\end{equation}
where $\textsc{Readout}(\cdot)$ is a permutation-invariant pooling.

\subsection{Optimization Objective}
Following classic MoE works \citep{noisy_gating, GMOE, MoG}, to prevent the gating network from consistently selecting a few experts, causing expert load imbalance during the training process, we introduce an expert importance loss as:
\begin{align}  \label{Important Loss}
    \mathrm{Importance}(\mathcal{V}) =& \sum_{v\in\mathcal{V}}\sum_{m\in\mathcal{M}}E_m^v,\\
    \mathcal{L}_{importance}(\mathcal{V}) = \mathrm{CV}&(\mathrm{Importance}(\mathcal{V}))^2,
\end{align}   
where $\mathrm{Importance}(\mathcal{V})$ represents the sum of each node's expert scores across all nodes, $\mathrm{CV}(\cdot)$ denotes the coefficient of variation, which encourages a more uniform positive distribution. Therefore, $\mathcal{L}_{importance}$ discourages expert collapse by minimizing disparity. The final loss function combines both task-specific loss and expert balance loss as follows:
\begin{equation} \label{final loss}
    \mathcal{L} = \mathcal{L}_{task} + \beta\cdot \mathcal{L}_{importance},
\end{equation}
where $\beta$ is a hyperparameter controlling the trade-off between two losses.

\section{Theoretical Analysis}
In this section, we analyze the expressive power of \ourmethod from the aspect of Subgraph Weisfeiler-Lehman Test \citep{SWL, GNN_AK} and the computational complexity of the proposed model.

\subsection{Expressive Power of \ourmethod}
To motivate our analysis, we begin by reviewing the classical 1-WL test and its subgraph-based generalization.

\begin{definition}[\textbf{1-WL Node Refinement}]
Let $G=(\mathcal{V}, \mathcal{E})$ be a graph, and let $c_v^{(0)}$ denote the initial color (label) of node $v$. At each iteration $t\geq 0$, the 1-WL test updates the node color by aggregating the multiset of neighbor colors:
\[
c_v^{(t+1)} = \operatorname{hash}\left(c_v^{(t)}, \multiset{c_u^{(t)} \mid u \in \mathcal{N}(v)}\right).
\]
After $T$ iterations, the multiset $\multiset{c_v^{(T)} \mid v \in \mathcal{V}}$ serves as the graph’s fingerprint. Two graphs are distinguished if their fingerprints differ.
\end{definition}

Modern MPNNs are known to be bounded by 1-WL \citep{1WL_upbound}, which limits their ability to distinguish certain non-isomorphic graphs. This is due to the fact that 1-WL only aggregates neighborhood color multisets without considering the structural relations among neighbors.
To enhance discrimination, SWL replaces the 1-hop neighborhood multiset with a richer subgraph view as follows:

\begin{definition}[\textbf{SWL Node Refinement}]
    Subgraph Weifeiler-Lehman generalizes the 1-WL graph isomorphism test algorithm by replacing the color refinement at iteration $t$ with $c^{(t+1)}_v =
\operatorname{hash}\bigl(G^{(t)}_v\bigr)$, where $G^{(t)}_v=(\mathcal V_v,\mathcal E_v,c_v^{(t)})$ is the induced subgraph rooted at node $v$.
\end{definition} \label{def:hswl}

In \ourmethod, each node $v$ is first associated with a topology-aware subgraph $G_v$ extracted via anonymous walk-based policy. Then, the model computes a vector $\Phi(G_v) := \left[\mathcal{K}^{(p)}(G_v, H_i)\right]_{i=1}^N$ of random walk kernel similarities between $G_v$ and a set of learnable hidden graphs $\{H_i\}_{i=1}^N$, followed by an injective MLP mapping. This operation can be interpreted as a learned hashing mechanism applied to $G_v$:
\[
c_v^{(t+1)} = \rho\left(\Phi(G_v)\right) = \rho\left( \left[\mathcal{K}^{(p)}(G_v, H_i)\right]_{i=1}^N \right),
\]
where $\rho$ is the MLP. To make this statement precise, we establish the following proposition and further give our corollary.

\begin{proposition} \label{sufficient hidden graphs}
Let $\mathcal{H}=\{H_i\}_{i=1}^N$ be a sufficiently large set of hidden graphs, and let $\mathcal{K}^{(p)}$ be a positive-definite random walk kernel that distinguishes all non-isomorphic subgraphs of bounded size. If $\rho$ is injective, then the representation $c_v^{(t+1)}$ computed by \ourmethod is injective over the set of $G_v$, i.e., \ourmethod simulates one round of SWL node refinement.
\end{proposition}


\begin{corollary} \label{no less than SWL}
Let $G$ and $G'$ be two graphs, and let $\pi$ denote a fixed subgraph extraction policy. If the SWL test distinguishes $G$ and $G'$ under policy $\pi$, then \ourmethod also distinguishes $G$ and $G'$ under the same policy, provided the conditions in \cref{sufficient hidden graphs} hold.
\end{corollary}

\cref{no less than SWL} shows that the expressiveness of \ourmethod is more powerful than SWL. The proof of \cref{sufficient hidden graphs} and \cref{no less than SWL} is given in Appendix A.2 and A.3.

\subsection{Complexity Analysis}
The complexity of \ourmethod consists of two main parts: anonymous walk-based subgraph extraction, and random walk kernel with hidden graphs. We denote the set of anonymous walks of length $l$ starting from node $v$ in graph $G$ as $\Omega^l(G, v)$, therefore, the complexity for sampling anonymous walks is $\mathcal{O}(nl\cdot\vert \Omega^l(G, v) \vert)$ where $n$ is the total number of nodes. For computing the RWK in \cref{derivation of K}, since the adjacency matrix is stored as a sparse matrix with $m$ non-zero entries, each subgraph expert calculation takes a computation time of $\mathcal{O}(Pf(Ns(s+n) + m))$, where $P$ is the maximum length of random walks, $f$ is the feature dimension, $N$ is hidden graphs number, and $s$ is the hidden graph size. The training procedure of \ourmethod is given in Appendix B.1.

\begin{table*}[ht]
    \centering
    \small
    \setlength{\tabcolsep}{1mm}
    \caption{Graph classification accuracies and deviations ($\%\pm\mathrm{std}$) of the proposed model and the baselines on 8 real-world graph datasets. The best results are in \textbf{bold} and the second-best results are \underline{underlined}.}
    \resizebox{0.95\textwidth}{!}{
    \begin{tabular}{cl|cccccccc}
        \toprule
       & Methods  & \textbf{MUTAG} & \textbf{D\&D} & \textbf{NCI1} & \textbf{PROTEINS} & \textbf{ENZYMES} & \textbf{IMDB-BINARY} & \textbf{IMDB-MULTI} & \textbf{REDDIT-BINARY} \\
        \midrule
        \midrule
        \parbox[t]{4pt}{\multirow{3}{*}{\rotatebox[origin=c]{90}{GK}}} 
        & GL  & 80.8 $\pm$ 6.4  & 75.4 $\pm$ 3.4  & 61.8 $\pm$ 1.7 & 71.6 $\pm$ 3.1 & 25.1 $\pm$ 4.4 & 63.3 $\pm$ 2.7 & 39.6 $\pm$ 3.0 & 76.6 $\pm$ 3.3   \\
       & SP  & 80.2 $\pm$ 6.5 & 78.3 $\pm$ 4.0 & 66.3 $\pm$ 2.6 & 71.9 $\pm$ 6.1 & 38.3 $\pm$ 8.0 & 57.5 $\pm$ 5.4 & 40.5 $\pm$ 2.8 & 75.5 $\pm$ 2.1  \\
       & WL  & 84.6 $\pm$ 8.3 & 78.1 $\pm$ 2.4 & \textbf{84.8 $\pm$ 2.5} & 73.8 $\pm$ 4.4 & 50.3 $\pm$ 5.7 & 72.8 $\pm$ 4.5 & 51.2 $\pm$ 6.5 & 74.9 $\pm$ 1.8  \\
        \midrule
        \parbox[t]{4pt}{\multirow{4}{*}{\rotatebox[origin=c]{90}{MPNN}}} 
       & GCN  & 74.3 $\pm$ 8.1 & 72.1 $\pm$ 0.3 & 80.2 $\pm$ 2.0 & 75.5 $\pm$ 1.6 & 57.8 $\pm$ 0.7 & \underline{74.0 $\pm$ 3.4} & \underline{51.9 $\pm$ 3.8} & 68.3 $\pm$ 1.1  \\
       & GAT &  86.4 $\pm$ 5.7 & 73.1 $\pm$ 3.4 & 82.4 $\pm$ 1.1 & 76.8 $\pm$ 1.7 & \underline{62.5 $\pm$ 6.1} & 72.0 $\pm$ 2.7 & 43.5 $\pm$ 3.5 & 90.8 $\pm$ 1.3 \\
      &  GIN  & 84.7 $\pm$ 6.7 & 75.3 $\pm$ 2.9 & 80.0 $\pm$ 1.4 & 73.3 $\pm$ 4.0 & 59.6 $\pm$ 4.5 & 71.2 $\pm$ 3.9 & 48.5 $\pm$ 3.3 & 89.9 $\pm$ 1.9  \\
       & SAGE  & 83.6 $\pm$ 9.6  & 72.9 $\pm$ 2.0 & 76.0 $\pm$ 1.8 & 73.0  $\pm$ 4.5 & 58.2 $\pm$ 6.0 & 68.8 $\pm$ 4.5 & 47.6 $\pm$ 3.5 & 84.3 $\pm$ 1.9  \\
        \midrule
        \parbox[t]{4pt}{\multirow{7}{*}{\rotatebox[origin=c]{90}{GK-based NN}}} 
       & RWNN  & 88.1 $\pm$ 4.8 & 76.9 $\pm$ 4.6 & 73.2 $\pm$ 2.0 & 74.1 $\pm$ 2.8 & 57.4 $\pm$ 4.9 & 70.6 $\pm$ 4.4 & 48.8 $\pm$ 2.9 & 90.3 $\pm$ 1.8  \\
       & KerGNN  & 72.7 $\pm$ 0.9 & 78.9 $\pm$ 3.5 & 76.3 $\pm$ 2.6 & 75.5 $\pm$ 4.6 & 55.0 $\pm$ 5.0 & 73.7 $\pm$ 4.0 & 50.9 $\pm$ 5.1 & 82.0 $\pm$ 2.5  \\
      &  GNN-AK  & 91.3 $\pm$ 7.0 & 80.3 $\pm$ 0.7 & 79.1 $\pm$ 2.0 & \underline{77.1 $\pm$ 5.7} & \textbf{62.5 $\pm$ 1.8} & 68.8 $\pm$ 3.2 & 44.5 $\pm$ 1.3 & 88.3 $\pm$ 1.2  \\
      &  RWK$^+$CN  & 83.0 $\pm$ 6.4 & 74.8 $\pm$ 2.4 & 72.3 $\pm$ 1.3 & 76.2 $\pm$ 1.2 & 44.0 $\pm$ 2.7 & 71.6 $\pm$ 1.0 & 49.2 $\pm$ 0.6 & 77.5 $\pm$ 0.6  \\
      &  GIP  & \underline{92.1 $\pm$ 3.3} & 79.5 $\pm$ 0.5 & 75.2 $\pm$ 1.5 & 76.8 $\pm$ 1.8 & 60.7 $\pm$ 2.6 & 71.1 $\pm$ 2.0 & 41.7 $\pm$ 1.5 & 82.7 $\pm$ 1.0  \\
      &  GKNN-GL  & 85.2 $\pm$ 2.3 & 78.6 $\pm$ 2.6 & 71.5 $\pm$ 1.2 & 75.4 $\pm$ 1.1 & 58.7 $\pm$ 3.2 & 69.9 $\pm$ 1.4 & 45.7 $\pm$ 1.2 & 89.2 $\pm$ 1.4  \\
      &  GKNN-WL  & 85.7 $\pm$ 2.8 & \underline{81.2 $\pm$ 2.5} & 73.6 $\pm$ 1.3 & 74.9 $\pm$ 1.1 & 59.0 $\pm$ 2.6 & 69.7 $\pm$ 2.2 & 47.9 $\pm$ 1.8 & \underline{91.4 $\pm$ 1.3} \\
        \midrule
      &  \ourmethod  & \textbf{92.7 $\pm$ 2.5} & \textbf{83.7 $\pm$ 1.6} & \underline{83.7 $\pm$ 1.7} & \textbf{78.6 $\pm$ 4.4} & \textbf{62.5 $\pm$ 1.8} & \textbf{83.0 $\pm$ 3.2} & \textbf{52.6 $\pm$ 1.1} & \textbf{92.1 $\pm$ 1.4}  \\
        \bottomrule
    \end{tabular} 
    } \label{graph classification: real world}
\end{table*}

\section{Experiments}
We conduct comprehensive experiments to validate the effectiveness and generalization of \ourmethod. The ablation study is then given to help us better understand our model. 

\subsection{Experiment Settings}
\paragraph{Datasets.} We conduct experiments on both graph-level and node-level tasks across a wide range of graph datasets. Their details and statistics are presented in Appendix B.2.
\begin{itemize}
    \item \textbf{Graph Classification:} We choose both real-world and synthetic graph datasets to evaluate our model. For the real-world ones, we use 8 publicly available graph classification datasets, including 5 biological and chemical compounds: MUTAG \citep{MUTAG}, D\&D \citep{DD}, NCI1 \citep{NCI1}, PROTEINS, ENZYMES \citep{PROTEINS_ENZYMES}, and 3 social interaction datasets: IMDB-BINARY, IMDB-MULTI, and REDDIT-BINARY \citep{IMDB_REDDIT}. For the synthetic ones, we include GraphFive and GraphCycle. GraphFive consists of five local structure classes: caveman, cycle, grid, ladder, and star. GraphCycle consists of two classes: Cycle and Non-Cycle.
    \item \textbf{Node Classification:} We choose 9 widely used node classification benchmarks, including 3 homophilic and 6 heterophilic graphs. We choose the most common citation networks \citep{GCN}: Cora, Citeseer, and Pubmed as our homophilic datasets. We consider Wikipedia networks (Chameleon and Squirrel), Actor co-occurrence network (Actor), and WebKB (Cornell, Texas, and Wisconsin) \citep{Geom-GCN} as heterophilic datasets.
\end{itemize}

\paragraph{Compared Methods.}
We compare the proposed model with the following baselines, which can be categorized into three categories: (1) Graph kernel methods: graphlet kernel (GL) \citep{graphletkernel}, shortest path kernel (SP) \citep{shortestpathkernel} and Weisfeiler-Lehman subtree kernel (WL) \citep{WLKernel}. (2) widely used MPNN-based GNNs: GCN \citep{GCN}, GAT \citep{GAT}, GIN \citep{GIN}, and GraphSAGE \citep{GraphSAGE}. (3) Graph kernel-based GNNs: RWNN \citep{RWNN}, KerGNN \citep{KerGNNs}, GNN-AK \citep{GNN_AK}, RWK$^+$CN \citep{RWK}, GIP \citep{GIP}, GKNN \citep{GKNN}.

\paragraph{Setup.} To make a fair comparison with other baselines, for graph classification datasets, we follow the 10-fold cross-validation for model assessment and use the same splits as described in \citep{FairComparison}. For node classification datasets, we use the public splits for homophilous graphs, and each heterophilous graph is split into 60\%/20\%/20\% for training, validation, and testing, respectively. 
The implementation details are presented in Appendix B.3.


\begin{table}[htbp]
    \centering
    \small
    \setlength{\tabcolsep}{1mm}
    \caption{Graph classification accuracies and F1 scores ($\%\pm\mathrm{std}$) on synthetic datasets. The best results are in \textbf{bold} and the second-best results are \underline{underlined}.}
    \resizebox{0.90\linewidth}{!}{
    \begin{tabular}{l|cc|cc}
    \toprule
    Datasets $\rightarrow$ & \multicolumn{2}{c}{\textbf{GraphFive}} & \multicolumn{2}{c}{\textbf{GraphCycle}} \\
    Methods $\downarrow$ & Acc. & F1 & Acc. & F1 \\
    \midrule
    SP & 56.7 $\pm$ 1.3  & 52.6 $\pm$ 1.0  & 79.8 $\pm$ 1.0  & 73.5 $\pm$ 0.9  \\
    WL & 60.1 $\pm$ 0.4  & \underline{56.3 $\pm$ 0.3}  & 81.1 $\pm$ 0.8  & 76.4 $\pm$ 0.8  \\
    \midrule
    GCN & 59.0 $\pm$ 2.3 & 53.8 $\pm$ 0.9 & 81.2 $\pm$ 1.1 & 73.0 $\pm$ 1.0 \\
    SAGE & 59.5 $\pm$ 0.5 & 51.2 $\pm$ 0.5 & 79.8 $\pm$ 1.1 & 71.2 $\pm$ 2.6 \\
    \midrule
    RWNN & 58.8 $\pm$ 1.5 & 53.7 $\pm$ 1.0 & 80.5 $\pm$ 1.7 & \underline{78.5 $\pm$ 2.8} \\
    KerGNN & 57.9 $\pm$ 0.5 & 49.7 $\pm$ 0.1 & 79.3 $\pm$ 0.8 & 71.8 $\pm$ 0.6 \\
    GIP & \underline{60.4 $\pm$ 1.8} & 55.9 $\pm$ 2.5 & \textbf{82.5 $\pm$ 1.2} & 77.9 $\pm$ 5.7 \\
    \ourmethod & \textbf{62.3 $\pm$ 1.2} & \textbf{57.4 $\pm$ 0.9} & \underline{82.3 $\pm$ 1.8} & \textbf{78.7 $\pm$ 0.8} \\
    \bottomrule
    \end{tabular}}
    \label{graph classification: synthetic datasets}
\end{table}

\subsection{Performance Evaluation}
\paragraph{Graph Classification.} 
The graph classification results on real-world and synthetic graph datasets are shown in \cref{graph classification: real world} and \cref{graph classification: synthetic datasets}, respectively. We observe that traditional graph kernels remain competitive, occasionally even outperforming MPNN-based GNNs on datasets with relatively small average graph sizes, such as MUTAG, NCI1, IMDB-BINARY, and IMDB-MULTI. However, on datasets with larger average numbers of nodes, MPNN-based models exhibit advantages in both classification accuracy and performance stability. Meanwhile, GK-based GNNs inherit the expressive power of classical graph kernels and thus achieve strong overall performance on graph classification benchmarks. A similar trend can also be found on the synthetic datasets, further supporting the generality of these findings. Among GK-based GNNs, \ourmethod consistently achieves the best or near-best results over 10 real-world and synthetic graph datasets. We attribute this to the MoE design, which enables our model to dynamically select and integrate informative subgraphs, effectively capturing diverse local patterns in different types of graphs.

\begin{table*}[ht]
    \centering
    \small
    \setlength{\tabcolsep}{1mm}
    \caption{Node classification accuracies on real-world graph datasets ($\%\pm\mathrm{std}$). The best results are in \textbf{bold} and the second-best results are \underline{underlined}. - means Out of Resources, either time or GPU memory.}
    \resizebox{0.90\textwidth}{!}{
    \begin{tabular}{ll|ccc cccccc}
        \toprule
        & Methods & \textbf{Cora} & \textbf{Citeseer} & \textbf{Pubmed} & \textbf{Chameleon} & \textbf{Squirrel} & \textbf{Actor} & \textbf{Cornell} & \textbf{Texas} & \textbf{Wisconsin} \\
        \midrule
        \midrule
        \parbox[t]{4pt}{\multirow{4}{*}{\rotatebox[origin=c]{90}{MPNN}}}
       & GCN  & 80.8 $\pm$ 0.5 & 70.8 $\pm$ 0.5 & 79.0 $\pm$ 0.3 & 59.8 $\pm$ 2.6 & 36.9 $\pm$ 1.3 & 30.3 $\pm$ 0.8 & 53.8 $\pm$ 8.6 & 54.1 $\pm$ 4.4 & 50.4 $\pm$ 7.6  \\
       & GAT & 83.0 $\pm$ 0.7 & 71.9 $\pm$ 1.4 & 85.3 $\pm$ 0.6 & 54.7 $\pm$ 2.0 & 30.6 $\pm$ 2.1 & 27.4 $\pm$ 0.9 & 55.0 $\pm$ 5.6 & 52.2 $\pm$ 6.6 & 54.3 $\pm$ 5.6 \\
       & GIN  & 81.9 $\pm$ 0.6 & 68.1 $\pm$ 0.6 & \underline{85.7 $\pm$ 0.3} & 56.2 $\pm$ 1.9 & 35.8 $\pm$ 1.3 & 23.6 $\pm$ 0.8 & 38.7 $\pm$ 5.2 & 29.0 $\pm$ 4.4 & 45.7 $\pm$ 4.7  \\
       & SAGE  & \textbf{87.7 $\pm$ 1.8} & \underline{76.7 $\pm$ 0.6} & \textbf{87.1 $\pm$ 0.7} & 55.1 $\pm$ 1.8 & 39.8 $\pm$ 1.9 & 24.2 $\pm$ 1.4 & 71.0 $\pm$ 4.8 & 67.4 $\pm$ 3.1 & 73.9 $\pm$ 6.8  \\
        \midrule
        \parbox[t]{4pt}{\multirow{7}{*}{\rotatebox[origin=c]{90}{GK-based NN}}}
       & RWNN & 71.8 $\pm$ 2.1 & 68.9 $\pm$ 0.5 & 64.6 $\pm$ 0.3 & 46.3 $\pm$ 2.7  & 29.8 $\pm$ 2.4 & 34.1 $\pm$ 0.3 & 62.8 $\pm$ 5.2 & 61.5 $\pm$ 7.5  & 75.4 $\pm$ 6.2  \\
       & KerGNN  & 72.9 $\pm$ 0.7 & 70.3 $\pm$ 0.4 & 68.1 $\pm$ 0.2 & 48.2 $\pm$ 2.5 & 33.9 $\pm$ 1.6 & 34.6 $\pm$ 0.4 & 64.9 $\pm$ 3.9 & 79.9 $\pm$ 7.1 & \underline{87.5 $\pm$ 4.6}  \\
       & GNN-AK  & 58.9 $\pm$ 0.3 & 50.1 $\pm$ 1.6 & 73.8 $\pm$ 0.6 & 56.8 $\pm$ 1.1 & 34.5 $\pm$ 1.1 & 35.8 $\pm$ 2.4 & \underline{72.4 $\pm$ 1.4} & \underline{83.2 $\pm$ 1.6}  & 84.3 $\pm$ 2.0\\
       & RWK$^+$CN  & 72.5 $\pm$ 1.7 & \textbf{76.7 $\pm$ 0.2} & 68.9 $\pm$ 0.5 & 47.4 $\pm$ 1.3 & 35.1 $\pm$ 1.6 & \underline{36.0 $\pm$ 0.2} & 72.3 $\pm$ 3.1 & 69.4 $\pm$ 6.0 & 73.5 $\pm$ 2.4 \\
       & GIP  & 76.6 $\pm$ 1.5 & 61.6 $\pm$ 0.1 & 77.5 $\pm$ 0.4 & \textbf{66.6 $\pm$ 0.7} & \underline{56.3 $\pm$ 1.1} & 24.9 $\pm$ 1.9 & 54.1 $\pm$ 5.3 & 82.4 $\pm$ 6.8 & 78.4 $\pm$ 2.0  \\
       & GKNN-GL  & - & - & - & - & - & - & - & - & - \\
       & GKNN-WL  & - & - & - & - & - & - & 54.1 $\pm$ 0.1 & 54.1 $\pm$ 0.0 & 49.2 $\pm$ 1.0 \\
        \midrule
       & \ourmethod  & \underline{83.9 $\pm$ 2.0} & 74.3 $\pm$ 0.4 & 79.8 $\pm$ 0.3 & \underline{64.5 $\pm$ 1.1} & \textbf{58.8 $\pm$ 0.7} & \textbf{39.6 $\pm$ 1.1} & \textbf{73.0 $\pm$ 3.4} & \textbf{83.7 $\pm$ 3.8} & \textbf{87.5 $\pm$ 4.3} \\
        \bottomrule
    \end{tabular} 
    } \label{node classification}
\end{table*}

\paragraph{Node Classification.}
The node classification results are shown in \cref{node classification}. We find that MPNN models generally perform well in homophilous graphs, where neighboring nodes are likely to share the same label. However, their performance tends to degrade under heterophily, where aggregating features from neighbors with different labels may dilute the target node's embedding. Conversely, GK-based models often perform better in heterophilous graphs, where the local structural patterns provide useful discriminative signals beyond label similarity. With the help of MoE design, our model significantly improves upon GK-based GNNs in both scenarios. Especially on heterophilous graphs, \ourmethod further amplifies the advantages of GK-based approaches.

\begin{figure}[htbp]
    \centering
    \includegraphics[width=0.90\linewidth]{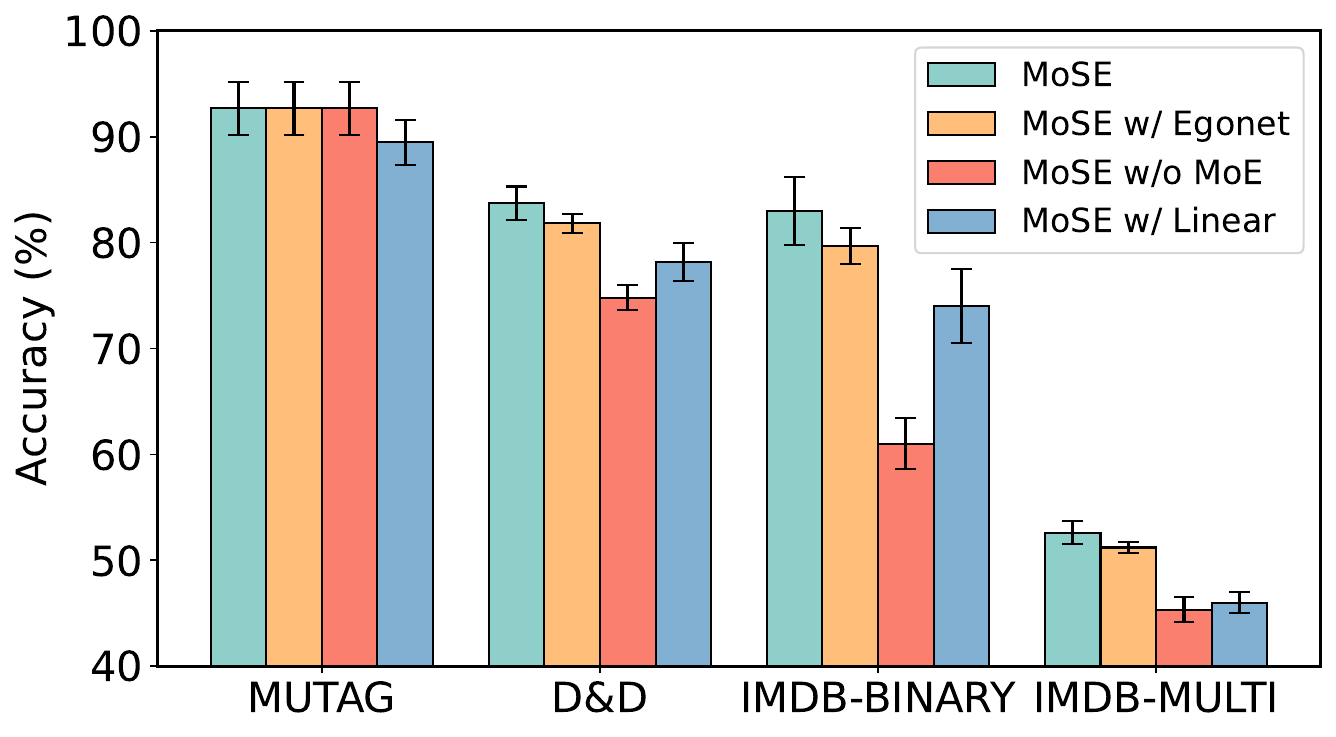}
    \caption{Ablation study with different \ourmethod variants.}
    \label{fig: ablation study}
\end{figure}

\subsection{Ablation Study and Visualization}
To better understand the contribution of each component in our model, we conduct ablation studies on several variants of \ourmethod. Specifically, we examine the effect of the subgraph extraction strategy, removing the MoE mechanism, and substituting the gating with a linear one. The results are shown in \cref{fig: ablation study}. Due to the small size of MUTAG, only the gating mechanism produces a noticeable performance difference. On other datasets, we find that the MoE component has the largest impact on performance, confirming its importance for capturing diverse subgraph semantics. Furthermore, subgraph-aware gating also contributes significantly by enabling the model to better associate input subgraphs with suitable experts. Lastly, the anonymous walk-based subgraph extraction is capable of isolating more discriminative local patterns.

\begin{figure}[hbtp]
    \centering
    \includegraphics[width=0.90\linewidth]{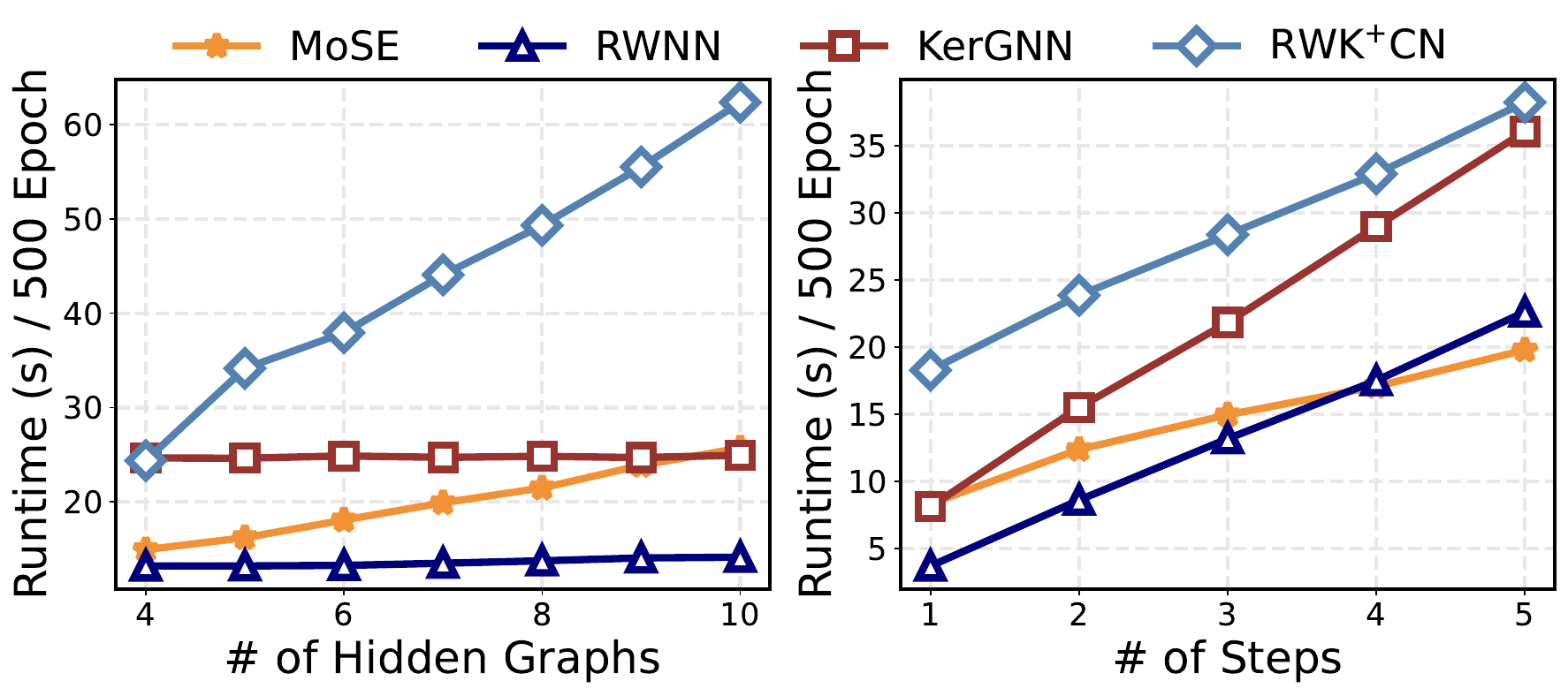}
    \caption{Running time comparison under different settings.}
    \label{fig:runtime comparison}
\end{figure}

We also evaluate the time complexity of our model under different configurations. \cref{fig:runtime comparison} reports the runtime per 500 epochs when varying the number of hidden graphs and RWK max steps, compared against other RWK-based models. The results show that the runtime of \ourmethod remains within a reasonable range, and scales smoothly with the number of steps. This confirms that our model offers an affordable trade-off between expressiveness and efficiency. 

\begin{figure}[htbp]
    \centering
    \includegraphics[width=0.85\linewidth]{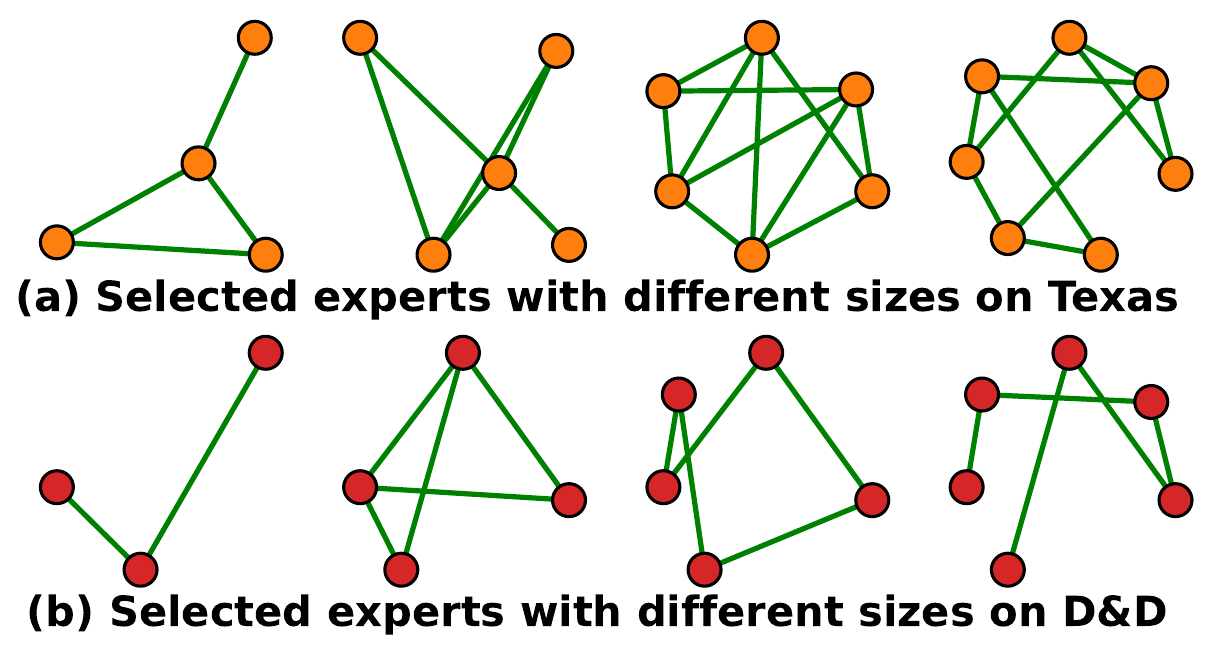}
    \caption{Subgraph experts visualization.}
    \label{fig: visualization}
\end{figure}

Furthermore, the visualization of the learned subgraph experts, presented in \cref{fig: visualization}, showcases the diversity and interpretability of our model. Compared to the case study in Introduction, our experts exhibit fewer redundant nodes while preserving key structural signals. The hyperparameter sensitivity analysis and additional visualizations on more datasets can be found in the Appendix B.4 and B.5.

\section{Related Work}
\paragraph{Hybrid Graph Kernels and GNNs.}
Graph kernels and GNNs are both fundamental tools for graph learning. Graph kernels, rooted in classical machine learning, provide interpretability by measuring graph similarities. In contrast, GNNs excel at efficient message passing, enabling flexible and scalable representation learning. While many works have attempted to bridge graph kernels and GNNs, existing approaches can be categorized into two aspects: using GNNs to enhance or design kernels \citep{hazan2015steps, GNTK},
and incorporating kernel principles into GNN design \citep{SCKN, KerGNNs, RWK, GCKN}. 
Our work falls into the second one, with a particular focus on integrating the widely used random walk kernel into GNNs to enhance their structural pattern awareness.

\paragraph{Graph Mixture of Experts.}
The Mixture of Experts (MoE) is a long-standing machine learning paradigm \citep{Longhistory1, Longhistory2} that trains multiple specialized experts to cooperatively enhance training efficiency and model performance. While MoE has been widely used in fields such as large language models \citep{switch_transformer}, video recognition \citep{MEID}, its application to graph remains relatively underexplored. Recently, a few pioneering efforts have extended MoE to the graph domain. GMoE \citep{GMOE} uses GNN with different hop numbers as experts to learn hop information. MoG \citep{MoG} incorporates multiple sparsifiers to personalize unique sparsity level and pruning criteria for each node. GraphMoRE \citep{GraphMoRE} introduces diverse Riemannian experts to embed the graph into manifolds with varying curvatures. In graph fairness learning, G-Fame \citep{G-Fame} is proposed as a plug-and-play method to learn distinguishable embeddings with multiple experts capturing different aspects of knowledge.

\section{Conclusion}
In this paper, we focus on subgraph pattern learning and propose \ourmethod framework for flexibility and adaptability across diverse graph tasks. By leveraging multiple subgraph experts with a newly designed subgraph extraction strategy and routing mechanism, our model captures meaningful local structures with controllable complexity. Extensive experimental results demonstrate that our model achieves a favorable balance between performance, efficiency, and interpretability.
\bibliography{aaai2026}
\clearpage
\appendix
\section{Theoretical Analysis} \label{Theoretical Analysis}
\subsection{Proof of Proposition 1}
Before we proceed with the proof, we first restate the relevant definition and proposition here for convenience. We also provide an illustrative example of anonymous walks in \cref{fig:Anonymous Walks Examples} to facilitate understanding.
\begin{figure}[htbp]
    \centering
    \includegraphics[width=1\linewidth]{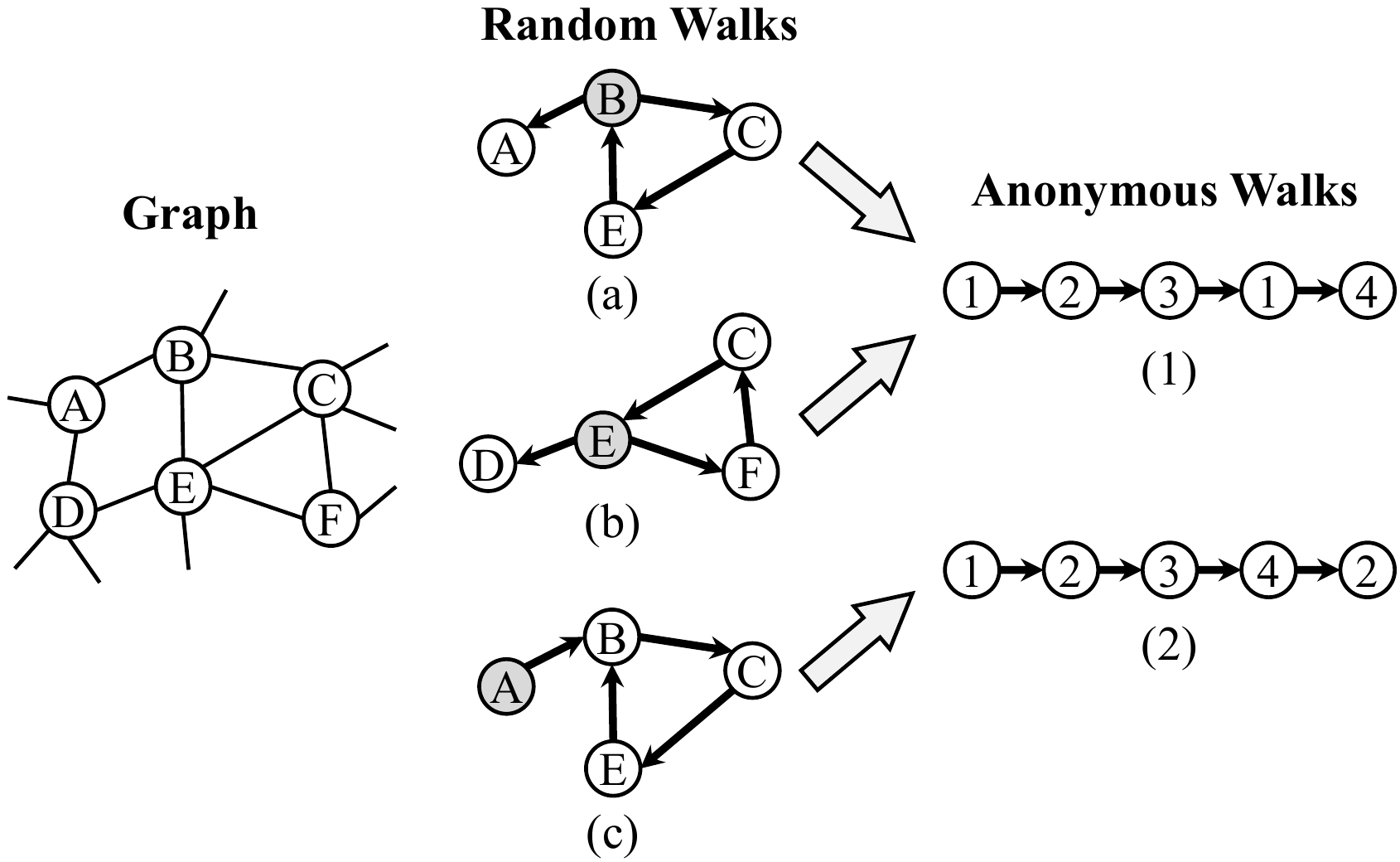}
    \caption{Anonymous walks examples with the starting node colored in grey. Random walks (a) and (b) share the same anonymous walk type, even though the node sets are different. The same node collections may refer to different anonymous walk types as walk (a) and (c).}
    \label{fig:Anonymous Walks Examples}
\end{figure}

\begin{definition}[\textbf{Anonymous Walk}]
Given a random walk $w=(v_0, v_1, \ldots, v_L)$ on a graph, the corresponding anonymous walk is defined as a sequence of integers $a=(\gamma_0, \gamma_1, \ldots, \gamma_L)$, where $\gamma_i=\min\mathrm{pos}(w, v_i)$, and $\mathrm{pos}(w, v_i)$ returns the position at which the node $v_i$ first appears in the walk $w$. We denote the mapping of a random walk $w$ to the anonymous walk $a$ by $w\mapsto a$.
\end{definition}

\begin{proposition} \label{AW reconstruct Ego Graph}
Let $G=(\mathcal{V}, \mathcal{E})$ be a graph and $v \in \mathcal{V}$. A sufficiently rich distribution $\mathcal{D}_l$ over length-$l$ anonymous walks starting at $v$ with $l = \mathcal{O}(|\mathcal{E}'|)$, is sufficient to reconstruct the ego graph $\mathcal{B}(v, h) = (\mathcal{V}', \mathcal{E}')$ centered at $v$ with radius $h$.
\end{proposition}
\begin{proof}
    Refer to the proof of Theorem 1 in \citep{AWequalKEGO}.
\end{proof}

\subsection{Proof of Proposition 2}
\begin{proposition} \label{sufficient hidden graphs}
Let $\mathcal{H}=\{H_i\}_{i=1}^N$ be a sufficiently set of hidden graphs, and let $\mathcal{K}^{(p)}$ be a positive-definite random walk kernel that distinguishes all non-isomorphic subgraphs of bounded size. If $\rho$ is injective, then the representation $c_v^{(t+1)}$ computed by \ourmethod is injective over the set of $G_v$, i.e., \ourmethod simulates one round of SWL node refinement.
\end{proposition}
\begin{proof}
Since $\mathcal{K}^{(p)}$ is positive-definite and $\mathcal{H}$ is sufficiently rich, for any pair of non-isomorphic subgraphs $G_v \not\simeq G_u$, there exists at least one hidden graph $H_j \in \mathcal{H}$ such that $\mathcal{K}^{(p)}(G_v, H_j) \neq \mathcal{K}^{(p)}(G_u, H_j)$. Hence, $\Phi(G_v) \neq \Phi(G_u)$. As $\rho$ is injective, this implies $c_v^{(t+1)} \neq c_u^{(t+1)}$. Therefore, the overall mapping $G_v \mapsto c_v^{(t+1)}$ is injective over the set of rooted subgraphs, consistent with the SWL refinement definition.
\end{proof}

\subsection{Proof of Corollary 1}
\begin{corollary} \label{no less than SWL}
Let $G$ and $G'$ be two graphs, and let $\pi$ denote a fixed subgraph extraction policy. If the SWL test distinguishes $G$ and $G'$ under policy $\pi$, then \ourmethod also distinguishes $G$ and $G'$ under the same policy, provided the conditions in \cref{sufficient hidden graphs} hold.
\end{corollary}

We prove the corollary by introducing the following definition \citep{SWL, SWL2} and lemma first.

\begin{definition} \label{def:3wl}
Let $C_1$ and $C_2$ be two color refinement algorithms, and denote $c_i(G)$, $i \in \{1, 2\}$ as the graph representation computed by $C_i$ for graph $G$. We say:
\begin{itemize}
    \item $C_1$ is \textit{no less powerful} than $C_2$, if there exist a pair of non-isomorphic graphs $G$ and $G'$ such that $c_1(G) \ne c_1(G')$ and $c_2(G) = c_2(G')$.
    \item $C_1$ is \textit{more powerful} than $C_2$, if for any pair of graphs $G$ and $G'$, if $c_2(G) \ne c_2(G')$, then $c_1(G) \ne c_1(G')$.
\end{itemize}
\end{definition}

\begin{lemma} \label{lem:1}
If the color of nodes in a graph cannot be further refined by $C_2$ when initialized with the stable coloring from $C_1$, then $C_1$ is more powerful than $C_2$.
\end{lemma}
The above lemma allows us to assess the relative expressive power of two color refinement algorithms on a single graph instead of requiring a pair of graphs for comparison. We now prove that \ourmethod is more powerful than SWL.

\begin{proof}
Given a graph $G$, let $c_v^{M}$ denote the stable color of node $v$ from \ourmethod. We take these stable colors as the initial colors of SWL, \ie, $c_v^{S, (0)}=c_v^{M}$. For the first SWL iteration, we have:
\begin{equation}
    c_v^{S, (1)}=\operatorname{hash}(G_v^{S,(0)})=\operatorname{hash}\left(\left(\mathcal V_v,\mathcal E_v,c_v^M\right)\right)
\end{equation}

Now consider two arbitrary nodes $v,u$.
\begin{enumerate}[(i)]
    \item $c^{M}_v=c^{M}_u.$ By \cref{sufficient hidden graphs}, the rooted subgraphs $G_v$ and $G_u$ are color-isomorphic under $c^{M}$, hence the injectivity of the hash yields $\operatorname{hash}(G^{S, (0)}_v)=\operatorname{hash}(G^{S,(0)}_u)$ and therefore $c^{S,(1)}_v=c^{S,(1)}_u$. Thus, SWL cannot split a color class that \ourmethod has already formed.
    \item $c^{M}_v\neq c^{M}_u.$ Again by \cref{sufficient hidden graphs} the subgraphs $G_v$ and $G_u$ are not color-isomorphic, so $\operatorname{hash}(G^{(0)}_v)\neq\operatorname{hash}(G^{(0)}_u)$ and consequently $c^{S,(1)}_v\neq c^{S,(1)}_u$. Hence, SWL never merges two distinct \ourmethod\ classes.
\end{enumerate}


\end{proof}

\newpage
\section{Experimental Supplementary}
\subsection{Training Procedure of \ourmethod}
\begin{algorithm}[htbp]
    \caption{Training Procedure}
    \label{training algorithm}
    \SetKwInOut{Input}{\textbf{Input}}
    \SetKwInOut{Output}{\textbf{Output}}
    \Input{Extracted Subgraphs $G_v=(\mathcal{V}_v, \mathcal{E}_v)$ for $v\in\mathcal{V}$; Subgraph Experts $E=\{E_s:s=1, \dots, K\}$ where $E_s=\{H_1^s, \ldots, H_N^s\}$ with hidden graphs}
    \Output{Node Embeddings $\bm{h}(v)$ for $v\in\mathcal{V}$, Learned hidden graphs}
    \For{$epoch = 1, ..., M$}{
        \For{$v\in\mathcal{V}$}{
        Calculate the expert scores and select experts via Eqs.(7)-(9)\;
        Compute MoE loss $\mathcal{L}_{importance}$ via Eq. (15)\;
        \For{$p=1, ..., P$}{
            Calculate the $p$-step random walk kernel between $G_v$ and $H_i$ in selected experts via Eq. (11)\;
            Concatenate the kernel values and map the results as $\bm{h}_s(v)$ via Eq. (10)\;
        }
        Combine the node embedding with expert scores as $\bm{h}(v)$ via Eq. (12)\;
        \textcolor{gray}{(Readout the graph embedding $\bm{h}(G)$ for graph-level task)\;}
        }
        Compute the task loss $\mathcal{L}_{task}$\;
        Final loss $\mathcal{L}=\mathcal{L}_{task}+\beta\cdot\mathcal{L}_{importance}$\;
        Back propagation, update parameters\;
        }
\end{algorithm}

\subsection{Dataset Details and Statistics}
We here introduce the details of the datasets used in our work.

\begin{figure}[htbp]
    \centering
    \includegraphics[width=1\linewidth]{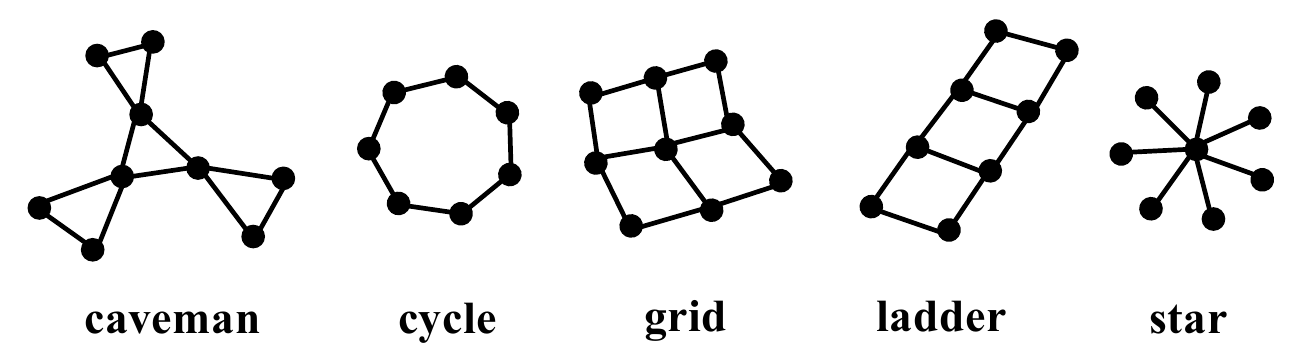}
    \caption{Illustration of the 5 types of graphs in synthetic dataset GraphFive}
    \label{fig:GraphFive}
\end{figure}

\paragraph{Graph Classification Datasets.} We choose 10 datasets, including 8 real-world graphs and 2 synthetic graphs. 

For real-world graphs, MUTAG \citep{MUTAG}, D\&D \citep{DD}, NCI1 \citep{NCI1}, PROTEINS, and ENZYMES \citep{PROTEINS_ENZYMES} are bio/chemo-informatics datasets. MUTAG consists of 188 mutagenic aromatic and heteroaromatic nitro compounds. The task is to predict whether or not each chemical compound has mutagenic effects on the Gram-negative bacterium Salmonella typhimurium. D\&D contains over a thousand protein structures. Each protein is a graph whose nodes correspond to amino acids. The task is to predict if a protein is an enzyme or not. NCI1 contains more than four thousand chemical compounds screened for activity against non-small cell lung cancer and ovarian cancer cell lines. PROTEINS consists of proteins represented as graphs where vertices are secondary structure elements, and there is an edge between two vertices if they are neighbors in the amino-acid sequence or in 3D space. The task is to classify proteins into enzymes and non-enzymes. ENZYMES contains 600 protein tertiary structures obtained from the BRENDA enzyme database. Each enzyme is a member of one of the Enzyme  Commission top level enzyme classes. IMDB-BINARY, IMDB-MULTI, and REDDIT-BINARY \citep{IMDB_REDDIT} are social interaction datasets. IMDB-BINARY and IMDB-MULTI are created from IMDB, an online database of information related to movies and television programs. The graphs contained in the two datasets correspond to movie collaborations. The vertices of each graph represent actors/actresses and two vertices are connected by an edge if the corresponding actors/actresses appear in the same movie. The task is to predict which genre an ego-network belongs to. REDDIT-BINARY contains graphs that model the social interactions between users of Reddit. Each graph represents an online discussion thread. The task is to classify graphs into either communities or subreddits. 

For synthetic datasets, we follow \citep{GIP} to generate GraphFive and GraphCycle. For GraphFive, we first generate 8~15 Barabási-Albert graphs as communities, each containing 10~200 nodes. Then, we connect the generated BA graphs in pre-defined five shapes: caveman, cycle grid, ladder, and star. To connect nodes in different clusters, we randomly add edges with a probability ranging from 0.05 to 0.15. We give an illustration of the 5 types of  graphs in \cref{fig:GraphFive}. For GraphCycle, we first generate 8-15 Barabási-Albert graphs as communities, each containing 10-200 nodes. Then, we connect the generated BA graphs in pre-defined two shapes: Cycle and Non-Cycle. To connect nodes in different clusters, we randomly add edges with a probability ranging from 0.05 to 0.15.
The statistics of these datasets are shown in \cref{tab: Statistics of graph classification datasets}.

\begin{table*}[htbp]
\caption{Statistics of the graph classification datasets} 
    \centering
    \resizebox{\textwidth}{!}{
    \begin{tabular}{c|cccccccccc}
    \toprule
    \textbf{Dataset} & MUTAG & D\&D & NCI1 & PROTEINS & ENZYMES & IMDB-BINARY & IMDB-MULTI & REDDIT-BINARY & GraphFive & GraphCycle \\
    \midrule
    \textbf{\# Graphs} & 188 & 1,178 & 4,110 & 1,113 & 600 & 1,000 & 1,500 & 2,000 & 5,000 & 2,000 \\
    \textbf{Type} & Biochem & Biochem & Biochem & Biochem & Biochem & Social & Social & Social & Synthetic & Synthetic\\
    \textbf{Max \# Nodes} & 28 & 5,748 & 111 & 620 & 126 & 136 & 89 & 3,782 & 225 & 513  \\
    \textbf{Min \# Nodes} & 10 & 30 & 3 & 4 & 2 & 12 & 7 & 6 & 133 & 122 \\
    \textbf{Avg \# Nodes} & 17.9 & 284.3 & 29.9 & 39.1 & 32.6 & 19.8 & 13.0 & 429.6 & 174.8 & 299.7 \\
    \textbf{Max \# Edges} & 33 & 14,267 & 119 & 1,049 & 149 & 1,249 & 1,467 & 4,071 & 1,480 & 2,346 \\
    \textbf{Min \# Edges} & 10 & 63 & 2 & 5 & 1 & 26 & 12 & 4 & 682 & 618\\
    \textbf{Avg. \# Edges} & 19.8 & 715.7 & 32.3 & 72.8 & 62.1 & 96.5 & 65.9 & 497.8 & 1033.5 & 1,641.5 \\
    \textbf{\# Class} & 2 & 2 & 2 & 2 & 6 & 2 & 3 & 2 & 5 & 2 \\
    \bottomrule
    \end{tabular}} \label{tab: Statistics of graph classification datasets}  
\end{table*}

\paragraph{Node Classification Datasets.} We choose 9 datasets, including 3 homophilous graphs and 6 heterophilous graphs. For homophilous graphs, we choose the most common citation networks \citep{GCN}: Cora, Citeseer, and Pubmed as our homophilic datasets. In these citation networks, nodes represent articles and edges represent the citation relationship between them. Node features are the word vectors and labels are the scientific fields of the papers. For heterophilous graphs, we consider Wikipedia networks (Chameleon and Squirrel), Actor co-occurrence network (Actor), and WebKB (Cornell, Texas, and Wisconsin) \citep{Geom-GCN} as small-scale heterophilic datasets. Wikipedia networks consist of web pages and mutual links between them. Node features are the informative nouns on these pages while the labels correspond to the monthly traffic of the page. Actor co-occurrence network is an induced subgraph of the film-director-actor-writer network. Each node denotes an actor with features representing keywords in Wikipedia. Each edge denotes the collaborations, and node labels are the types of actors. Nodes in WebKB represent web pages and edges are hyperlinks between them. Node features are the bag-of-words vectors and the labels are identities. The statistics of these datasets are shown in \cref{tab:Statistics of node classification datasets}

\begin{table}[htbp] 
    \caption{Statistics of the node classification datasets} 
    \centering
    \resizebox{\linewidth}{!}{
    \begin{tabular}{l|ccccc}
         \toprule
         \textbf{Dataset} & \textbf{\# Nodes} & \textbf{\# Edges} & \textbf{\# Features} & \textbf{\# Classes} & \textbf{Homo.}  \\
         \midrule
         Cora & 2,708 & 10,556 & 1,433 & 7 & 0.825 \\
         Citeseer & 3,327 & 9,104 & 3,703 & 6 & 0.717 \\
         Pubmed & 19,717 & 88,648 & 500 & 3 & 0.792 \\
         \midrule
         Chameleon & 2,277 & 36,051 & 2,325 & 5 & 0.249 \\
         Squirrel & 5,201 & 216,933 & 2,089 & 5 & 0.219 \\
         Actor & 7,600 & 29,926 & 932 & 5 & 0.206 \\
         Cornell & 183 & 295 & 1,703 & 5 & 0.106 \\
         Texas & 183 & 309 & 1,703 & 5 & 0.103 \\
         Wisconsin & 251 & 499 & 1,703 & 5 & 0.134 \\
         \bottomrule
    \end{tabular}\label{tab:Statistics of node classification datasets}}
\end{table}

\subsection{Experimental Setup}
For all datasets, we apply the Adam optimizer, and the hyper-parameters that we tune include the learning rate, the dropout rate, the number of subgraph experts, the number and size of hidden graphs in each subgraph expert, the step length of the random walk kernel, and the hidden dimension. For GK-based GNNs, we extract node representations from the layer right before the readout operation and attach a two-layer MLP for the node classification task. 
The model is implemented with the deep learning library PyTorch 2.2.0 and PyG 2.4.0, running on a Linux server with the Intel(R) Xeon(R) Silver 4210R CPU @ 2.40GHz, the NVIDIA RTX 3090 GPU (24 GB), and 128 GB memory.

\subsection{Hyperparameter Sensitivity Analysis}
To better understand the robustness and stability of \ourmethod, we conduct sensitivity analyses with respect to two critical hyperparameters: the number of hidden graphs per expert, and the number of steps used in the Random Walk Kernel (RWK). All experiments in this section are trained for 50 epochs to control training time, and only the hyperparameter on the x-axis is varied, with all other model settings fixed.

\begin{figure}[htbp]
    \centering
    \includegraphics[width=1\linewidth]{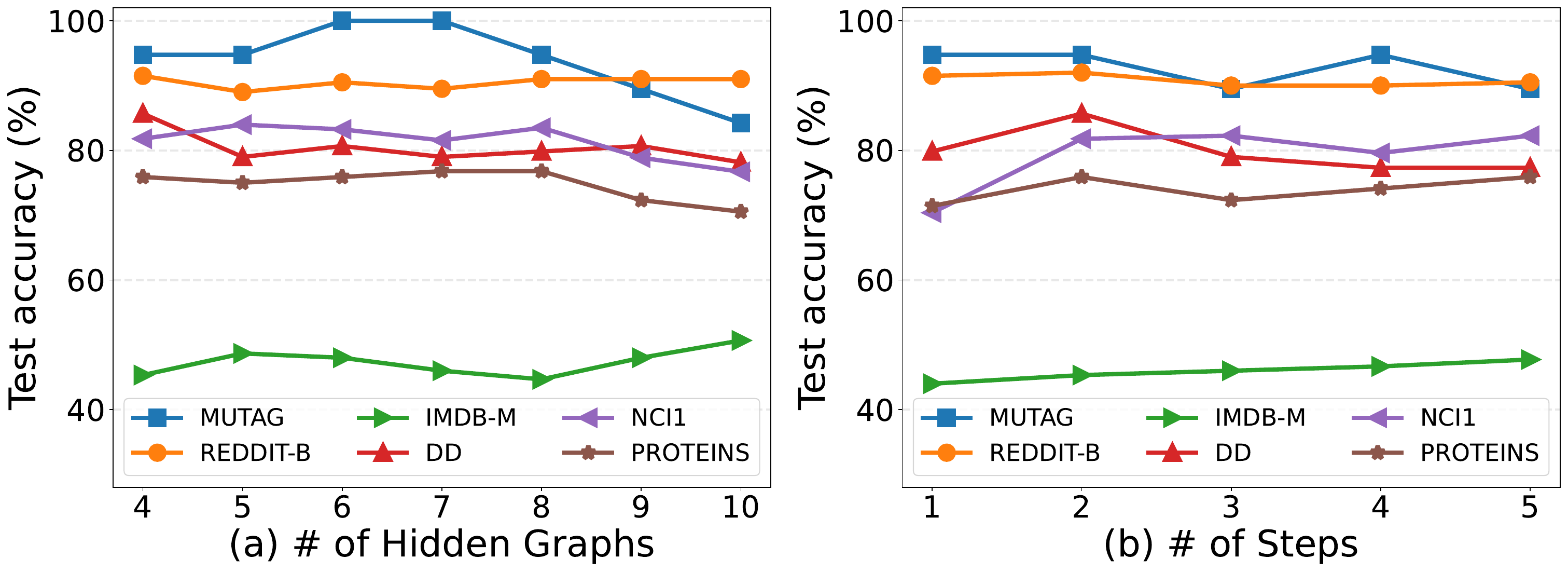}
    \caption{Sensitivity analysis of \ourmethod with respect to (a) the number of hidden graphs per expert and (b) the number of random walk steps.}
    \label{fig:Sensitivity Analysis}
\end{figure}

As shown in \cref{fig:Sensitivity Analysis}(a), increasing the number of hidden graphs generally improves model performance across most datasets, especially in the early range. This is because more hidden graphs allow the model to better capture diverse structural patterns within each expert. However, we observe diminishing returns or even performance drops beyond a certain threshold, as on PROTEINS and NCI1, likely due to overfitting or increased optimization difficulty. These results highlight the importance of balancing representation capacity and generalization.
We also investigate how the number of steps used in random walks affects the model performance in \cref{fig:Sensitivity Analysis}(b). Increasing the number of steps enriches the expressiveness of extracted walk patterns. For most datasets, performance improves up to 3 or 4 steps and then saturates or slightly drops. This suggests that excessively long walks may introduce noise or redundant information, especially for smaller graphs like MUTAG or DD. Thus, a moderate step size provides a good trade-off between expressiveness and efficiency.

\subsection{Additional Visualizations}
To further demonstrate the behavior of our model, we provide additional visualizations of the learned subgraph experts on four datasets: Chameleon, IMDB-MULTI, MUTAG, and Wisconsin. For Chameleon and IMDB-MULTI in \cref{fig:visualization 1}, we visualize the selected subgraph experts as determined by the MoE routing mechanism. These examples highlight how different experts are activated for different subgraph structures, capturing diverse local patterns relevant to node or graph-level tasks.

\begin{figure}[htbp]
    \centering
    \includegraphics[width=1\linewidth]{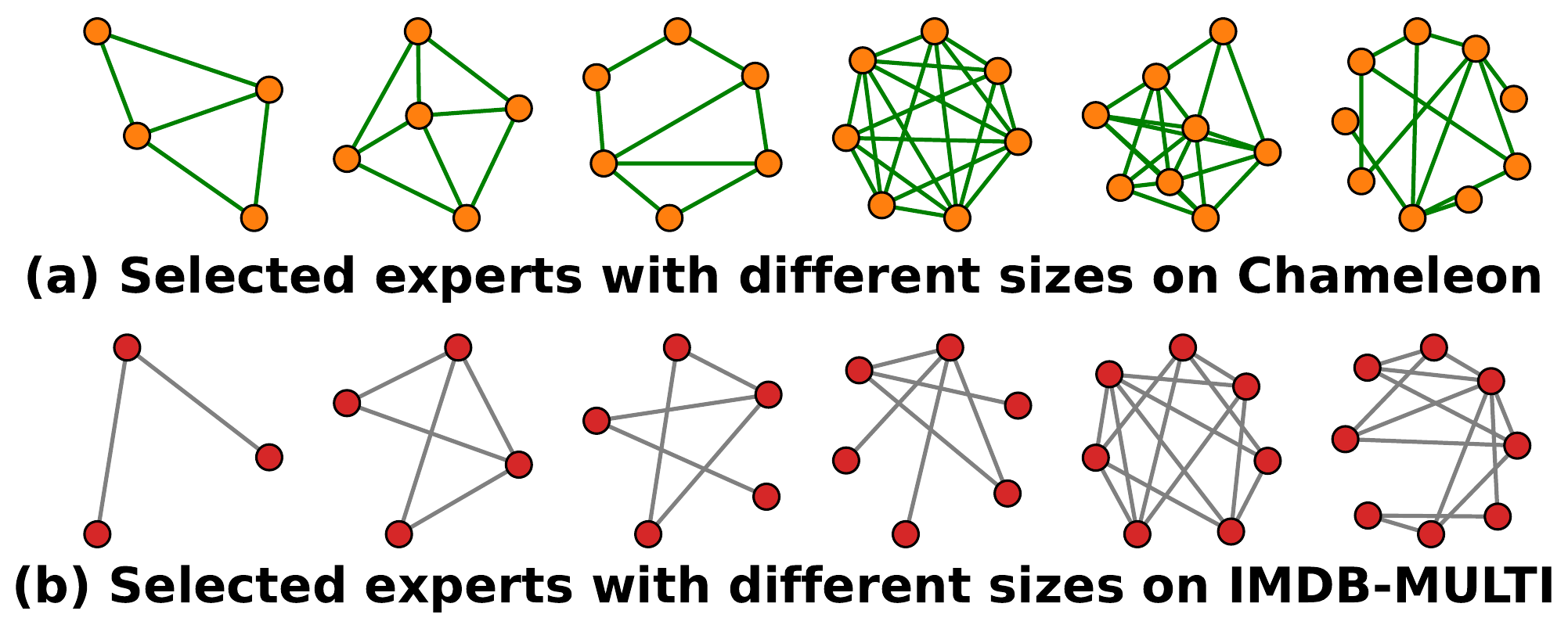}
    \caption{Selected subgraph experts visualization on Chameleon and IMDB-MULTI.}
    \label{fig:visualization 1}
\end{figure}

In contrast, for Wisconsin and MUTAG, we visualize all subgraph experts without filtering by the MoE selection in \cref{fig: MUTAG visualization} and \cref{fig: wisconsin visualization}. In these visualizations, each row corresponds to a different expert, and each column represents a hidden graph in that expert. This exhaustive presentation reveals the redundancy among hidden graphs: several experts produce similar or uninformative patterns that are eventually filtered out by the gating mechanism.
Interestingly, we observe that on simpler datasets such as MUTAG, there is a higher degree of redundancy across experts, indicating that fewer unique subgraph patterns are needed to solve the task. Conversely, Wisconsin, which contains more nuanced structural information, exhibits less redundancy and benefits from a richer set of subgraph experts to capture its underlying characteristics.

\begin{figure*}[htbp]
    \centering
    \includegraphics[width=1\linewidth]{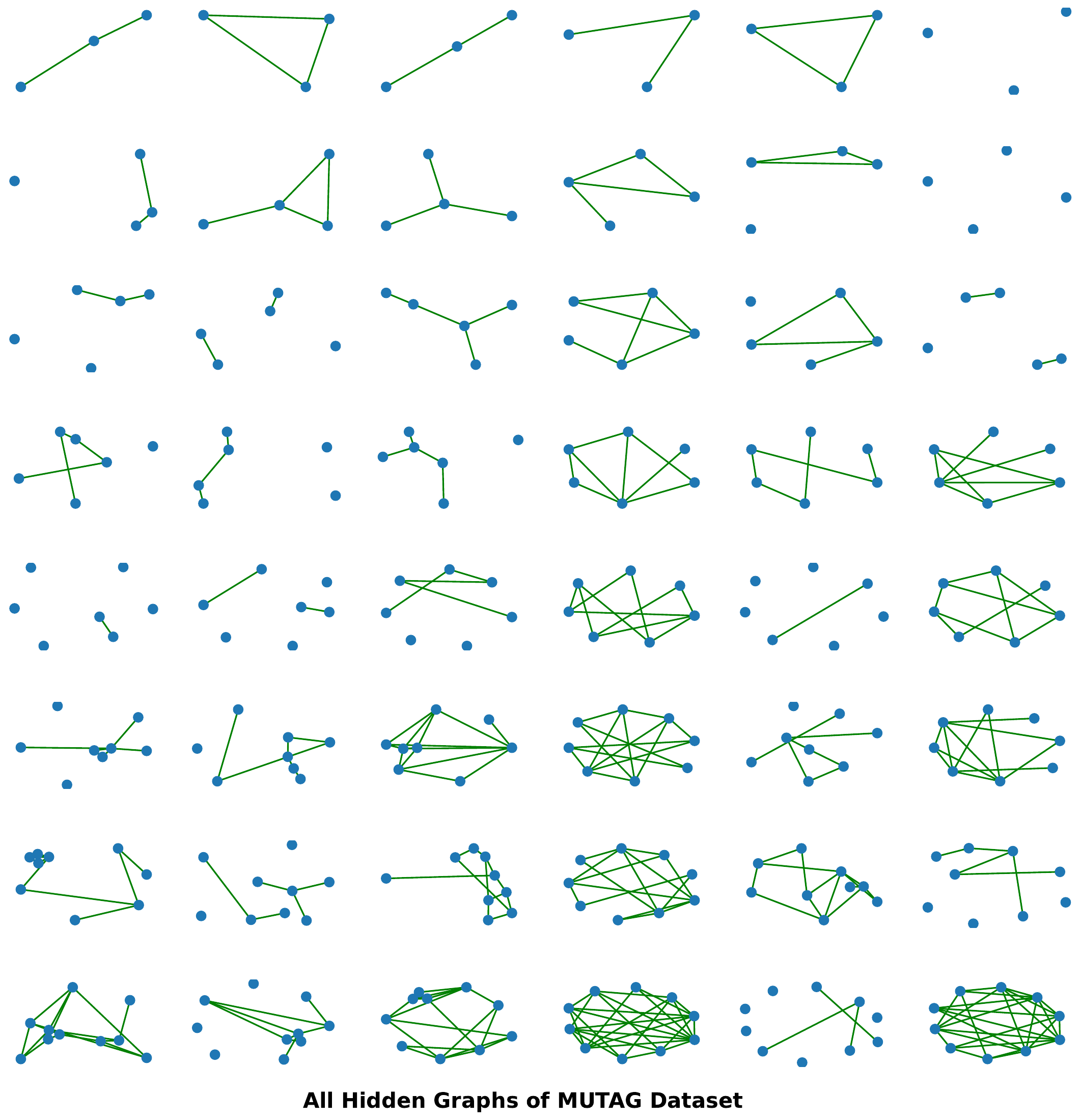}
    \caption{All hidden graphs of MUTAG, regardless of MoE selection. Each row corresponds to a different expert, and each column represents a hidden graph in that expert.}
    \label{fig: MUTAG visualization}
\end{figure*}

\begin{figure*}[htbp]
    \centering
    \includegraphics[width=1\linewidth]{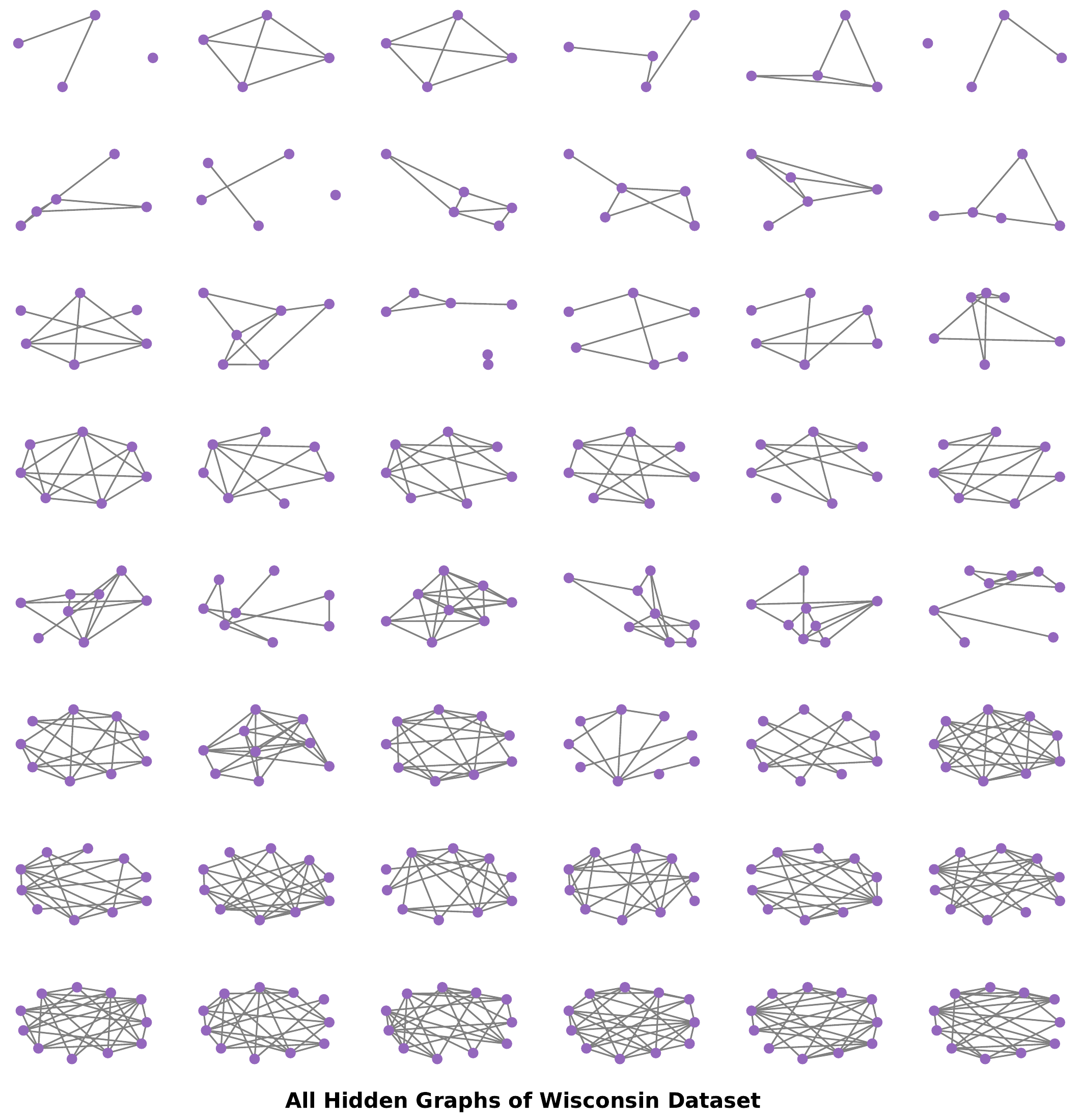}
    \caption{All hidden graphs of Wisconsin, regardless of MoE selection. Each row corresponds to a different expert, and each column represents a hidden graph in that expert.}
    \label{fig: wisconsin visualization}
\end{figure*}

\clearpage
\end{document}